\documentclass[oneside,english]{amsart}
\usepackage[foot]{amsaddr}
\usepackage{times}
\usepackage{geometry}
\usepackage[T1]{fontenc}
\usepackage[latin9]{inputenc}
\usepackage{amsthm}
\usepackage{amssymb}
\usepackage{dsfont}
\usepackage[all,cmtip]{xy}
\usepackage{hyperref}
\usepackage{stmaryrd}
\usepackage[autostyle]{csquotes}
\usepackage{mathtools,bm,enumitem}

\newtheorem{Theorem}{Theorem}[section]
\newtheorem{Lemma}[Theorem]{Lemma}
\newtheorem{Proposition}[Theorem]{Proposition}
\newtheorem{Corollary}[Theorem]{Corollary}
\theoremstyle{definition}
\newtheorem{Definition}[Theorem]{Definition}
\newtheorem{Example}[Theorem]{Example}
\theoremstyle{remark}
\newtheorem{Remark}[Theorem]{Remark}

\makeatletter
\numberwithin{equation}{section}
\numberwithin{figure}{section}

\makeatother

\newcommand{\EE}{\mathbb{E}}

\newcommand{\ZZ}{\mathcal{Z}}
\renewcommand{\AA}{\mathcal{A}}

\newcommand{\TT}{\mathcal{T}}

\newcommand{\NN}{\mathcal{N}}
\newcommand{\XX}{\mathcal{X}}
\newcommand{\YY}{\mathcal{Y}}
\newcommand{\RR}{\mathbb{R}}
\newcommand{\MM}{\mathcal{M}}
\newcommand{\WW}{\mathcal{W}}

\newcommand{\LL}{\mathcal{L}}

\newcommand{\ovec}{\operatorname{vec}}

\newcommand{\DD}{\mathcal{D}}

\newcommand{\ba}{\mathbf{a}}

\newcommand{\bc}{\mathbf{c}}
\newcommand{\bx}{\mathbf{x}}
\newcommand{\bw}{\mathbf{w}}
\newcommand{\by}{\mathbf{y}}
\newcommand{\bz}{\mathbf{z}}
\newcommand{\bb}{\mathbf{b}}
\newcommand{\bo}{\mathbf{o}}
\newcommand{\bT}{\mathbb{T}}
\newcommand{\bW}{\mathbf{W}}
\newcommand{\bB}{\mathbf{B}}
\newcommand{\bG}{\mathbf{G}}
\newcommand{\bI}{\mathbf{I}}
\newcommand{\bJ}{\mathbf{J}}
\newcommand{\bA}{\mathbf{A}}
\newcommand{\bZ}{\mathbf{Z}}
\newcommand{\bF}{\mathbf{F}}
\newcommand{\bOmega}{\mathbf{\Omega}}
\newcommand{\bPhi}{\mathbf{\Phi}}
\newcommand{\btau}{\bm{\tau}}
\newcommand{\bgamma}{\bm{\gamma}}

\DeclareMathOperator{\Exp}{Exp}

\DeclarePairedDelimiterX{\infdivx}[2]{(}{)}{%
	#1\;\delimsize\|\;#2%
}
\newcommand{\KLd}{\operatorname{KL}\infdivx}

\usepackage{babel}
\usepackage{tikz}
\title{A Coordinate-Free Construction of Scalable Natural Gradient}
\author{Kevin Luk$^*$}
\address[$*$]{Borealis AI}
\email{kevin.luk@borealisai.com}
\author{Roger Grosse$^\dagger$}
\address[$\dagger$]{University of Toronto, Vector Institute}
\email{rgrosse@cs.toronto.edu}

\begin{document}

\maketitle

\begin{abstract}
	Most neural networks are trained using first-order optimization methods, which are sensitive to the parameterization of the model. Natural gradient descent is invariant to smooth reparameterizations because it is defined in a coordinate-free way, but tractable approximations are typically defined in terms of coordinate systems, and hence may lose the invariance properties. We analyze the invariance properties of the Kronecker-Factored Approximate Curvature (K-FAC) algorithm by constructing the algorithm in a coordinate-free way. We explicitly construct a Riemannian metric under which the natural gradient matches the K-FAC update; invariance to affine transformations of the activations follows immediately. We extend our framework to analyze the invariance properties of K-FAC applied to convolutional networks and recurrent neural networks, as well as metrics other than the usual Fisher metric.
\end{abstract}

\section{introduction} \label{sec:intro}

Most neural networks are trained using stochastic gradient descent (SGD) \cite{Bottou_sgd}, or variants thereof which adapt step sizes for individual dimensions \cite{adagrad,Adam}. One well-known deficiency of SGD is that the updates are sensitive to the parameterization of the network. There are numerous tricks for reparameterizing network architectures so that they represent the same sets of functions, but in a friendlier coordinate system. Examples include replacing logistic activation functions with $\tanh$ \cite{GlorotBengio10}, whitening the inputs or activations \cite{natural_neural_networks,krizhevsky2009learning}, or centering the activations to have zero mean and/or unit variance \cite{RBM_centering,enhanced_gradient,Raiko_linear,batch_norm}. Such tricks can lead to large improvements in the speed of optimization.

Ideally, one would like to use an optimization algorithm which is invariant to such transformations of a neural network, in order to avoid the pathologies which the transformations are meant to remedy. Natural gradient descent \cite{Amari98} is a second-order optimization algorithm motivated by a key invariance property: to the first-order, its updates are invariant to smooth reparameterizations of a model. The natural gradient of a cost function can be seen as the gradient of the function on a Riemannian manifold (typically using the Fisher information metric \cite{AmariBook}), and the invariance properties of the algorithm follow directly from its definition in terms of differential geometric primitives.

There have been many attempts to apply natural gradient descent, or approximations thereof, to training neural networks \cite{Amari_MLP,TONGA,HF,FANG,KFAC15,desjardins2015natural,schulman2015trust}. The challenge is that the exact natural gradient is impractical to compute for large neural nets, because it requires solving a linear system whose dimension is the number of parameters (which may be in the tens of millions for modern networks). Unfortunately, tractable approximations to the natural gradient are typically defined in terms of particular coordinate representations, and therefore may lose the invariance properties which motivated natural gradient in the first place. For instance, diagonal approximations to natural gradient descent (e.g.~\cite{adagrad,Adam}) are not invariant to re-centering of the inputs. Ollivier \cite{Ollivier15} presented an approximation to natural gradient which is invariant to affine transformations of individual coordinates of the input, though this misses important classes of transformations such as whitening.

Kronecker-Factored Approximate Curvature (K-FAC) \cite{KFAC15} is an approximate natural gradient optimizer where the Fisher information matrix $\mathbf{F}$ is approximated as a block diagonal matrix with one block per layer of the network, and each block factorizes as the Kronecker product of small matrices. Because of the Kronecker structure, the approximate natural gradient can be computed with low overhead relative to ordinary SGD; K-FAC demonstrated significant speedups in training deep autoencoders \cite{KFAC15}, classification convolutional networks \cite{KFC,dist_KFAC}, recurrent networks~\cite{martens2018kronecker} and deep reinforcement learning~\cite{wu2017scalable}. The same Fisher matrix approximation has also led to significant improvements in modeling posterior uncertainty in Bayesian neural networks~\cite{ritter2018scalable,zhang2017noisy} and avoiding catastrophic forgetting~\cite{ritter2018online}. 

Although K-FAC does not satisfy the general invariance properties of natural gradient, it is still invariant to a broad and interesting class of reparameterizations: affine transformations of the activations in each layer \cite{KFAC15}. This was verified through linear algebraic manipulation of the update rules, but unfortunately the proofs yielded little insight into the algorithm or advice about how it can be extended. 

Here we take a different approach: we formulate K-FAC directly in terms of coordinate-free mathematical objects, so that the invariance properties follow immediately from the construction. Specifically, we view a neural network as a series of affine maps alternating with fixed nonlinear activation functions. The activations and pre-activations for each layer are viewed as abstract affine spaces, and the weights and biases of the network correspond to affine maps. The ordinary Fisher metric is a metric on this space $\WW$ of affine maps. Our contribution is a recipe to convert a metric on $\WW$ (whose coordinate representation is extremely large) into an approximate metric on $\WW$ (the \emph{``K-FAC metric''}), whose coordinate representation matches the K-FAC approximation. Hence, rather than view K-FAC as an approximation to the natural gradient under the Fisher metric, we view it as the exact natural gradient under the K-FAC metric. This entire construction is coordinate-free, so the invariance properties of K-FAC follow immediately.

We can contrast K-FAC's invariance properties with those of exact natural gradient descent. Since the exact natural gradient is derived in terms of a metric on a smooth manifold, the update is invariant to arbitrary smooth reparameterizations, but only up to the first-order. An update rule which achieves higher-order invariance with modest computational overhead was introduced recently in~\cite{song2018accelerating}. By contrast, we show global invariance to a more restricted class of reparameterizations. Our analysis imposes additional structure on the weight manifold $\WW$: the parameters are assumed to define affine maps between affine spaces. Choosing affine bases for the activations and pre-activations yields a natural affine basis for $\WW$. Therefore, the set of allowable reparameterizations for neural networks consists of affine change-of-basis transformations for the activations and pre-activations. This leaves out some unusual reparameterizations which exact natural gradient descent is invariant to, such as permuting the entries of the weight matrix. But it captures important classes of reparameterizations, such as whitening, normalization, and replacing the logistic activation function with $\tanh$. And in exchange for imposing the affine structure, we obtain \emph{global} invariance, not just first-order invariance.

Our framework easily enables some generalizations of the basic result. First of all, our construction applies to general \emph{pullback metrics}, where one places a metric on the network's output space and pulls it back to $\WW$. In addition to the Fisher metric, this also includes the Gauss-Newton and Generalized Gauss-Newton metrics \cite{Martens14}. The Gauss-Newton metric is defined in terms of particular output space parameterizations (e.g.~logits), which implies the derived K-FAC update is not invariant to reparameterizations of the outputs; however, our analysis shows it is invariant to affine reparameterizations of all other layers of the network. We also extend the invariance results to convolutional networks and recurrent neural networks through a straightforward application of our K-FAC metric construction.

\subsection{Organization of the paper.}~\mbox{} \label{subsec:organization}

We begin in Section~\ref{sec:background} with general background on the natural gradient and the K-FAC approximation. In the latter half of Section~\ref{sec:background}, we provide background on all of the mathematical machinery that we use later in the paper. In Section~\ref{sec:abstract-mlp}, we formulate the multilayer perceptron architecture in a coordinate-free manner, where the activations and pre-activations are considered as elements in affine spaces, and the weights and biases parameterize affine maps. Sections~\ref{subsec:independence metric} and~\ref{subsec:KFAC metric} form the heart of the paper: we show how to convert pullback metrics on $\WW$ into metrics whose coordinate representation matches the K-FAC approximation. The invariance properties of K-FAC follow immediately. Sections~\ref{sec:abstract-CNN} and~\ref{sec:abstract-RNN} extend our analysis to convolutional networks and recurrent neural networks, respectively. Both cases are straightforward applications of the tools developed in Section~\ref{sec:abstract-mlp}, illustrating the power and flexibility of our approach.

\subsection{Acknowledgments.}~\mbox{}
 
We are grateful to Francis Bischoff, Matt Johnson, Chia-Cheng Liu, and Sushant Sachdeva for clarifications and helpful discussions.

\begin{table}
	\begin{tabular}{ |c|c|c|  }
		\hline
		& Coordinate-independent & Coordinate-dependent \\
		\hline
		inputs & $\xi$ & $\bf{x}$ \\
		targets & $\upsilon$ & $\bf{y}$ \\
		activations & $\alpha_i$ & $\bf{a}_i$ \\
		homogenized activations & & $\mathbf{\bar{a}}_i$ \\
		pre-activations & $\zeta_i$ & $\bf{z}_i$ \\
		layerwise parameters (and biases) & $\omega_i$ & $(\mathbf{W}_i\ \mathbf{b}_i)$ \\
		homogenized layerwise parameters & & $\bar{\bW}_i$ \\
		activation functions & $\rho_i$ & $\phi_i$   \\
		network parameter & $\omega$ & $\bw$ \\
		network output & $f(\xi,\omega)$ & $f(\bx,\bw)$ \\
		loss function & $\LL(\upsilon,f(\xi,\omega))$ & $\LL(\by,f(\bx,\bw))$ \\
		objective function & $h(\omega)$ & $h(\bw)$ \\
		predictive distribution & $P_{\upsilon|\xi}(\omega)$ & $P_{\by|\bx}(\bw)$ \\
		predictive distribution density function & $p(\upsilon|\xi,\omega)$ & $p(\by|\bx,\bw)$ \\
		Fisher metric/matrix & $g_F(\omega)$ & $\bF(\bw)$ \\
		layerwise log-likelihood differential/gradient & $d\LL_{\omega_i}$ & $\DD\bar{\bW}_i$ \\
		log-likelihood differential/gradient & $d\LL_{\omega}$ & $\DD\bw$ \\
		K-FAC metric/matrix & $g_{\mathrm{KFAC}}(\omega)$ & $\hat{\bF}(\bw)$ \\
		tensor product/Kronecker product & $\otimes$ & $\otimes$ \\
		input space & $\XX$ & \\
		output space & $\YY$ & \\
		space of activations & $\AA_i$ & \\
		space of pre-activations & $\ZZ_i$ & \\
		layerwise weight space  & $\WW_i$ & \\
		weight space of network & $\WW$ & \\
		\hline
	\end{tabular}
	\caption{Since this paper involves the interplay of many coordinate-independent and coordinate-dependent objects, we summarize these notations here. Note that these notations are for Sections~\ref{sec:background} and~\ref{sec:abstract-mlp} where we work with MLPs. The notations in Sections~\ref{sec:abstract-CNN} and~\ref{sec:abstract-RNN} for convolutional networks and recurrent networks will be self-contained. As a general rule of thumb, we use boldface to symbolize coordinate-dependent objects and standard math font for coordinate-independent ones.}
\end{table}

\clearpage

\section{background} \label{sec:background}

\subsection{Algorithmic background}~\mbox{}

We first present an introduction to the (exact) natural gradient descent algorithm and show how the invariance properties of this algorithm are immediate from its coordinate-free formulation. We then provide an overview of the K-FAC algorithm following~\cite{KFAC15}.

\subsubsection{Natural gradient descent}~\mbox{} \label{subsubsec:NGD basics}

For simplicity, we define the natural gradient descent algorithm here in the context of multilayer perceptrons (MLPs), i.e., fully connected feed-forward networks. Given an input-target pair $(\bx,\by)$, let $f(\bx,\bw)$ denote the output and $\bw$ symbolize the parameter vector of the MLP. We would like to minimize the expected risk
\[
\EE_{\bx,\by}[\LL(\by,f(\bx,\bw))],
\]
where $\LL(\by,f(\bx,\bw))$ is the loss function measuring the disagreement between $\by$ and $f(\bx,\bw)$. The expectation above is taken with respect to a joint distribution over $(\bx,\by)$, such as the empirical distribution. For a training set $\mathcal{S}$ of pairs $(\bx_i,\by_i)$, the empirical risk is given by

\begin{equation} \label{eq:empirical risk function}
h(\bw) = \frac{1}{|\mathcal{S}|}\sum_{(\bx_i,\by_i)\in\mathcal{S}}\LL(\by_i,f(\bx_i,\bw)).
\end{equation}

Suppose that $f(\bx,\bw)$ determines parameters $\bz$ of the model's predictive distribution $R_{\by|\bz}$ over $\by$ and furthermore, we reparameterize this as $P_{\by|\bx}(\bw)=R_{\by|f(\bx,\bw)}$. Likewise, the density function of this distribution can be reparameterized as $p(\by|\bx,\bw)=r(\by|f(\bx,\bw))$. In addition, we take the loss function here to be the negative log-likelihood $\LL(\by,\bz)=-\log r(\by|\bz)$ and denote log-likelihood gradients $\nabla_{\bw}\LL(\by,f(\bx,\bw))$ by $\DD\bw$ ($\DD$ notation throughout remainder of the paper refers to log-likelihood gradients).  The Fisher information matrix $\bF(\bw)$ is defined as
\begin{equation} \label{eq:Fisher matrix}
\bF(\bw) = \EE_{\bx,\by}[(\DD\bw)(\DD\bw)^{\top}],
\end{equation}
where the expectation is taken over $P_{\by|\bx}(\bw)$ for $\by$ and over the data distribution for $\bx$. Since $\bF(\bw)$ is defined as the expectation of an outer product, $\bF(\bw)$ is always guaranteed to be a positive-semidefinite (PSD) matrix.

The natural gradient of the objective function $h(\bw)$ in Eqn.~\ref{eq:empirical risk function} is $\bF(\bw)^{-1}\nabla h(\bw)$. For a chosen learning rate $\epsilon>0$, the natural gradient descent algorithm~\cite{Amari98} minimizes $h(\bw)$ by using the natural gradient to update parameters of the network:
\begin{equation} \label{eq:NGD update rule}
\bw\leftarrow\bw-\epsilon \bF(\bw)^{-1}\nabla h(\bw).
\end{equation}

Natural gradient descent can be understood as a second-order optimization algorithm. As shown in~\cite{Martens14,pascanu2013revisiting}, for small $\delta>0$, the second-order Taylor expansion of the KL-divergence between $P_{\by|\bx}(\bw)$ and $P_{\by|\bx}(\bw+\delta)$ is
\[
\KLd{P_{\by|\bx}(\bw)}{P_{\by|\bx}(\bw+\delta)}\approx\frac{1}{2}\delta^{\top} \bF(\bw)\delta.
\]

For special cases where the predictive distribution $R_{\by|\bz}$ of the network corresponds to an exponential family with $\bf{z}$ representing natural parameters, $\bF(\bw)$ is exactly the generalized Gauss-Newton matrix~\cite{Martens14}. This matrix is often used as a curvature matrix for various second-order optimization methods; for example, in Hessian-free optimization~\cite{martens2010deep} or in Krylov subspace descent~\cite{vinyals2012krylov}.

\subsubsection{Invariance properties of natural gradient descent}~\mbox{} \label{subsubsec:NG invariances}

In addition to exploiting the local geometric structure of the space of predictive distributions, natural gradient descent possesses a key invariance property which does not hold for ordinary stochastic gradient descent (SGD): given two equivalent networks which are parameterized differently, after applying the natural gradient descent update to each, the resulting networks will be equivalent up to the first-order. The reason for this is that the natural gradient admits an intrinsic coordinate-free construction in terms of differential geometric primitives. 

We consider an abstract mathematical setting. This section uses some standard mathematical terminology for which the background is given in Section \ref{sec:math-background}.

Let $\MM$ be a Riemannian manifold with Riemannian metric given by $g$. For a smooth function $h:\MM\to\RR$ and a point $p\in\MM$, the differential $dh(p)$ is an abstract covector on $\MM$. To convert the covector $dh(p)$ into a tangent vector, we use the Riemannian metric $g$. By definition, $g(p)$ is a nondegenerate bilinear form which yields the linear isomorphism between the tangent space and the cotangent space: 
\[
g(p):T_p\MM\overset{\cong}{\to} T_p^*\MM.
\]
This is commonly referred to as the musical isomorphism in mathematical literature. The inverse $g(p)^{-1}$ gives a linear map the other way around,
\[
g(p)^{-1} : T_p^*\MM \overset{\cong}{\to} T_p \MM.
\]
Applying this isomorphism to $dh(p)$ yields the tangent vector $g(p)^{-1}dh(p)\in T_p\MM$. We call this tangent vector the natural gradient of $h$.

We apply this mathematical framework to the objects of our interest. First, let $\WW$ be a smooth manifold which characterizes the weight space of network parameters intrinsically. Let $\omega$ and $(\xi,\upsilon)$ be intrinsic versions of the parameter vector $\bw$ and the input-target pair $(\bx,\by)$ respectively. Now, $\WW$ can be endowed with the Fisher metric $g_F$ which is defined as
\begin{equation} \label{eq:Fisher metric on parameter space}
g_F(\omega) = \EE_{\xi,\upsilon}[d\LL_{\omega}\otimes d\LL_{\omega}],
\end{equation}
where $\LL_{\omega} = -\log p(\upsilon|\xi,\omega)$ is the abstract log-likelihood loss function. The expectation above is taken over the abstract predictive distribution $P_{\upsilon|\xi}(\omega)$ for $\upsilon$ and over the data distribution for $\xi$. Expressing this in a coordinate system, we have
\begin{equation} \label{eq:Fisher metric in coordinates computation}
\begin{aligned}
\llbracket g_F(\omega) \rrbracket & = \llbracket \EE_{\xi,\upsilon}[d\LL_{\omega}\otimes d\LL_{\omega}]\rrbracket \\
& = \EE_{\bx,\by}[\llbracket d\LL_{\omega}\otimes d\LL_{\omega}\rrbracket] \\
& = \EE_{\bx,\by}[(\DD\bw)(\DD\bw)^{\top}] \\
& = \bF(\bw),
\end{aligned}
\end{equation}
which is exactly the Fisher matrix as given earlier in Eqn.~\ref{eq:Fisher matrix}. 

As $\WW$ is a Riemannian manifold with Fisher metric $g_F$, the idealized gradient descent updates are given by~\cite{bonnabel2013stochastic}
\begin{equation} \label{eq:RGD update rule}
\omega\leftarrow\omega-\epsilon\mathrm{Exp}_{\omega}(g_F(\omega)^{-1}dh(\omega)),
\end{equation}
where $\Exp_{\omega}:T_{\omega}\WW\to\WW$ is the exponential map. This update rule is~\emph{exactly} invariant to all smooth reparameterizations of $\WW$ since it is entirely coordinate-free. However, such an algorithm is infeasible in practice as computing the exponential map is typically an intractable problem. Instead, it is much easier to work with the following abstract natural gradient update rule which uses a first-order approximation of the exponential map
\[
\omega\leftarrow\omega-g_F(\omega)^{-1}dh(\omega).
\]
Writing the above expression in coordinates, by Eqn.~\ref{eq:Fisher metric in coordinates computation}, is equivalent to the update rule in Eqn~\ref{eq:NGD update rule}. Since this is a first-order approximation of the update rule in Eqn~\ref{eq:RGD update rule}, invariance to smooth reparameterizations holds only up to first-order. Additional approximations to the exponential map are necessary to obtain higher-order invariances; we defer to~\cite{song2018accelerating} for a more detailed account of how this can be done.

\subsubsection{Kronecker-Factored Approximate Curvature (K-FAC)}~\mbox{}

We consider a MLP with $L$ layers. At each layer $i\in\{1,\dots,L\}$, the MLP computation is given as follows:
\begin{align*}
\bz_i & = \bW_i\ba_{i-1} + \bb_i \\
\ba_i & = \phi_i(\bz_i),
\end{align*}
where $\ba_{i-1}$ is an activation vector, $\bz_i$ is a pre-activation vector, $\bW_i$ is a weight matrix, $\bb_i$ is a bias vector, and $\phi_i:\RR\to\RR$ is an activation function. For convenience, we introduce homogeneous coordinates $\bar{\ba}_{i-1}^{\top} = [\bar{\ba}_{i-1}^{\top}\ 1]^{\top}$ and $\bar{\bW}_i = [\bW_i\ \bb_i]$. Then, the above computation can be rewritten as
\begin{equation} \label{eq:homogeneous feed-forward computation}
\begin{aligned}
\bz_i & = \bar{\bW}_i\bar{\ba}_{i-1} \\
\ba_i & = \phi_i(\bz_i).
\end{aligned}
\end{equation}
We concatenate all of the network parameters $\bar{\bW}_i$ into a single vector $\bw$,
\[
\bw = [\ovec(\bar{\bW}_1)^{\top}\ \ovec(\bar{\bW}_2)^{\top}\ \dots\ \ovec(\bar{\bW}_L)^{\top}]^{\top}.
\]
Here, $\ovec$ denotes the vectorization operator which stacks the columns of a matrix together to form a vector. The Fisher matrix for the MLP is a $L\times L$ block matrix $\bF({\bw})$ with each $(i,j)$-th block given by
\[
\bF(\bw)_{i,j} = \EE[\ovec(\DD\bar{\bW}_i)\ovec(\DD\bar{\bW}_j)^{\top}].
\]

Given an objective function $h(\bw)$, we can minimize this using natural gradient $\bF(\bw)^{-1}\nabla h(\bw)$ as explained previously. While natural gradient descent has desirable theoretical properties, it is not feasible for practical purposes: the major challenge lies in the difficulty of both storing $\bF(\bw)$ and solving linear systems involving $\bF(\bw)$ for large networks which may have millions of parameters. By making assumptions on the underlying probabilistic model structure, the Kronecker-Factored Approximate Curvature (K-FAC) method~\cite{KFAC15} approximates the Fisher matrix efficiently from a computation standpoint. We now give a brief overview of the K-FAC algorithm.

Consider the diagonal $(i,i)$ blocks of $\bF(\bw)$. Using backpropagation, the log-likelihood gradient $\DD\bar{\bW}_i = \DD\bz_i\bar{\ba}_{i-1}^{\top}$ and hence, we have $\ovec(\DD\bz_i\bar{\ba}_{i-1}^{\top})=\bar{\ba}_{i-1}\otimes\DD\bz_i$. Then, $\bF(\bw)_{i,i}$ can be rewritten as:
\begin{align*}
\bF(\bw)_{i,i} & = \EE[\ovec(\DD\bar{\bW}_i)\ovec(\DD\bar{\bW}_i)^{\top}] \\
& = \EE[(\bar{\ba}_{i-1}\otimes\DD\bz_i)(\bar{\ba}_{i-1}\otimes\DD\bz_i)^{\top}] \\
& = \EE[\bar{\ba}_{i-1}\bar{\ba}_{i-1}^{\top}\otimes\DD\bz_i\DD\bz_i^{\top}],
\end{align*}
where $\otimes$ denotes the Kronecker product of matrices. If the activations and pre-activation derivatives are approximated as statistically independent, this yields the following approximation $\hat{\bF}(\bw)_{i,i}$ to $\bF(\bw)_{i,i}$,
\begin{equation} \label{eq:KFAC decomposition blocks}
\hat{\bF}(\bw)_{i,i} = \EE[\bar{\ba}_{i-1}\bar{\ba}_{i-1}^{\top}]\otimes\EE[\DD\bz_i\DD\bz_i^{\top}] = \bA_{i-1}\otimes\bG_i,
\end{equation}
where $\bA_{i-1} = \EE[\bar{\ba}_{i-1}\bar{\ba}_{i-1}^{\top}]$ and $\bG_i = \EE[\DD\bz_i\DD\bz_i^{\top}]$ are second moment matrices of the activations and pre-activation derivatives respectively. The K-FAC approximation matrix $\hat{\bF}(\bw)$ to $\bF(\bw)$ is defined as
\begin{equation} \label{eq:KFAC matrix}
\hat{\bF}(\bw) = \left[\begin{array}{cccc}
\bA_{0}\otimes \bG_{1} &  &  & \mathbf{0}\\
& \bA_{1}\otimes \bG_{2}\\
&  & \ddots\\
\mathbf{0} &  &  & \bA_{L-1}\otimes \bG_{L}
\end{array}\right].
\end{equation}
To determine the inverse $\hat{\bF}(\bw)^{-1}$, we use the fact that Kronecker factors may be inverted in the following way: $(\bB\otimes\mathbf{C})^{-1} = \bB^{-1}\otimes\mathbf{C}^{-1}$. Thus, the approximate natural gradient using K-FAC, $\hat{\bF}(\bw)^{-1}\nabla h(\bw)$, can be computed as
\[
\hat{\bF}(\bw)^{-1}\nabla h(\bw)=\left[\begin{array}{c}
\ovec(\bG_{1}^{-1}(\nabla_{\bar{\bW}_{1}}h(\bw))\bA_{0}^{-1})\\
\vdots\\
\ovec(\bG_{L}^{-1}(\nabla_{\bar{\bW}_L}h(\bw))\bA_{L-1}^{-1})
\end{array}\right].
\]
Analogously to the natural gradient descent algorithm given earlier, K-FAC updates the parameters of the network according to the following update rule:
\begin{equation} \label{eq:KFAC update rule}
\bw\leftarrow\bw-\epsilon\hat{\bF}(\bw)^{-1}\nabla h(\bw).
\end{equation}

\subsubsection{Invariance properties of K-FAC}~\mbox{}

Since K-FAC uses the approximation $\hat{\bF}(\bw)$ rather than the Fisher matrix $\bF(\bw)$ itself, the invariance properties of natural gradient do not necessarily carry over to K-FAC. Instead, we consider the class of transformations given by the following transformed network
\begin{equation} \label{eq:transformed network computation}
\begin{aligned}
\bz_i^{\dagger} & = \bar{\bW}_i^{\dagger}\bar{\ba}_{i-1}^{\dagger} \\
\ba_i^{\dagger} & = \phi_i^{\dagger}(\bz_i^{\dagger}) \\
&  = \bOmega_i\phi_i(\bPhi_i\bz_i+\bm{\tau}_i)+\bgamma_i
\end{aligned}
\end{equation}
where $\bOmega_i$, $\bPhi_i$ are invertible matrices and $\btau_i$, $\bgamma_i$ are vectors. The transformed input is $\bar{\ba}_0^{\dagger} = \bar{\bOmega}_0\bar{\ba}_0$ where
\[
\bar{\bOmega}_0 = \left[\begin{array}{cc}
\bOmega_0 & \bgamma_0 \\
\mathbf{0} & 1
\end{array}\right],
\]
and the transformed output is $\ba_L^{\dagger} = f^{\dagger} (\bx^{\dagger},\bw^{\dagger})$ with $\bw^{\dagger}$ defined as
\[
\bw^{\dagger} = [\ovec(\bar{\bW}_1^{\dagger})^{\top}\ \ovec(\bar{\bW}_2^{\dagger})^{\top}\ \dots\ \ovec(\bar{\bW}_L^{\dagger})^{\top}]^{\top}.
\]
The original and transformed network are equivalent in terms of the functions they compute. We observe that the transformations given in Eqn.~\ref{eq:transformed network computation} encompasses a wide range of transformations. These include common deep learning tricks such as centering the activations to have zero mean and/or unit variance and replacing logistic sigmoid activation functions with $\tanh$. While K-FAC may not be invariant under smooth parameterizations of the model as in the case of natural gradient, the following theorem shows that it is invariant to the class of transformations given in Eqn.~\ref{eq:transformed network computation}.

\begin{Theorem}[Theorem 1 (rephrased here) in~\cite{KFAC15}]\label{KFAC invariances theorem}
Let $\NN$ be the network with parameter vector $\bw$ and activation functions $\{\phi_i\}^{L}_{i=1}$. Suppose that we have activation functions $\{\phi_i^{\dagger}\}^L_{i=1}$ as given in Eqn.~\ref{eq:transformed network computation}. Then, there exists a parameter vector $\bw^{\dagger}$ such that the transformed network $\NN^{\dagger}$ with parameter vector $\bw^{\dagger}$ and activation functions $\{\phi_i^{\dagger}\}^L_{i=1}$ computes the same function as $\NN$. Furthermore, the K-FAC updates are equivalent, in the sense that the resulting networks compute the same function.
\end{Theorem}

The proof of this theorem in~\cite{KFAC15} is dependent on the choice of a coordinate system for the network. Our central goal in this paper is to provide a coordinate-free construction of K-FAC; in particular, we like to construct a metric $g_{\mathrm{KFAC}}$ such that $\llbracket g_{\mathrm{KFAC}}(\omega)\rrbracket = \hat{\bF}(\bw)$. In this way, we may view the K-FAC update rule in Eqn.~\ref{eq:KFAC update rule} as a natural gradient update with respect to the K-FAC metric $g_{\mathrm{KFAC}}$. More importantly, by doing so, the invariance properties of K-FAC are immediately established in the same way as it was for exact natural gradient. 

\subsection{Mathematical background}~\mbox{} \label{sec:math-background}

As the coordinate-free construction of K-FAC requires mathematical machinery from both abstract linear algebra and differential geometry, we devote this section of the paper to introduce these mathematical tools. Furthermore, since we move from coordinate-independent to coordinate-dependent mathematical objects frequently in this paper, we set the notation $\llbracket\cdot\rrbracket$ here to mean choosing coordinates for an abstract object. 

\subsubsection{Vector spaces and tensor algebra.}~\mbox{} \label{subsubsec:vector spaces}

Let $V$ be a vector space. The dual space $V^{*}$ of $V$ is the set of all linear functionals on $V$ and this space itself admits the structure of a vector space. The direct sum $U\oplus V$ of two vector spaces $U$ and $V$ is a vector space where the set structure is the Cartesian product $U\times V$ and the addition and multiplication is given by
\begin{align*}
(u_1,v_1)+(u_2,v_2) & =(u_1+v_1,u_2+v_2) \\
c(u_1,v_1) &= (cu_1,cv_1),
\end{align*}
for $u_1,u_2\in U$, $v_1,v_2\in V$ and $c\in\RR$. 

We now introduce tensors on vector spaces. A $k$-tensor $T$ on $V$ is a multilinear function
\[
T:\underbrace{V\times\dots\times V}_{k\ \mathrm{copies}}\to\RR,
\]
We may think of $T$ as an element of the vector space $(V^*)^{\otimes k}$, the tensor product of the vector space $V^*$ with itself $k$-times. We delegate the definition of a tensor product of vector spaces to Appendix~\ref{Asubsec: tensor product of VS}. A $k$-tensor $T$ is symmetric if $T$ is a symmetric multilinear function. We work primarily with symmetric tensors in this paper. 

\subsubsection{Canonical isomorphisms}~\mbox{}

We describe the distinction between an isomorphism and a canonical isomorphism of vector spaces. An isomorphism between two vector spaces $U$ and $V$ is a bijection between $U$ and $V$ which preserves addition and scalar multiplication. A canonical isomorphism is a stronger concept, it is an isomorphism of vector spaces which is natural, in the sense that it does not depend on any choice of bases to define the isomorphism. For example, any two vector spaces of the same dimension are isomorphic to one another but the isomorphism may not be canonical. Consider a finite dimensional vector space $V$. There is an isomorphism between $V$ and its dual space $V^*$: given a choice of basis $\mathbf{e}_i$ for $V$, there is a dual basis $\mathbf{e}_i^*$ for $V^*$ and the map $\mathbf{e}_i\to \mathbf{e}_i^*$ is an isomorphism. However, this is not canonical as it depends on the choice of basis $\mathbf{e}_i$ for $V$. On the other hand, consider the evaluation map $\mathrm{ev}_v:V^*\to\RR$ defined by $\varphi\mapsto\varphi(v)$ where $\varphi\in V^*$, $v\in V$. The mapping $v\mapsto\mathrm{ev}_v$ then defines a canonical isomorphism from $V$ to its double dual space $V^{**}$.

\subsubsection{Affine algebra}~\mbox{}

A set $A$ is an affine space associated to the vector space $V$ if there is a mapping $A\times A\to V$ denoted by $(p,q)\in A\times A\mapsto\vec{pq}\in V$ satisfying the axioms
\begin{enumerate}
	\item for any $p,q,r\in A$, $\vec{pr}=\vec{pq}+\vec{qr}$
	\item for any $p\in A$ and for any $x\in V$ there is an unique $q\in A$ such that $x=\vec{pq}$.
\end{enumerate}
Intuitively, an affine space may be thought of as a vector space with no privileged origin. Suppose that we choose an origin point $o\in A$ and let $\{\mathbf{e}_1,\dots,\mathbf{e}_n\}$ be a basis for the associated vector space $V$. For any point $p\in A$, we can write $\vec{op}=\sum_{i=1}^{n}\bx_{i}(p)\mathbf{e}_{i}$. Here, $\{\bx_{1},\dots,\bx_{n}\}$ is a set of coordinate functions, or more simply,  a basis for $A$. If we have two bases $\{\bx_{1},\dots,\bx_{n}\}$ and $\{\by_{1},\dots,\by_{n}\}$, then they are related by $\mathbf{y=Bx+c}$ where $\mathbf{B}=[\mathbf{b}_{ij}]$ is an invertible $n\times n$ matrix and $\mathbf{c}=[\mathbf{c}_{i}]$ is a vector.

We now describe how to extend a change-of-basis on the affine space $A$ to the product space $A^K = A\times\dots\times A$. Let $\iota$ and $\kappa$ be two choices of affine bases on $A$, then 
\[
\llbracket(a_1,\dots,a_k)\rrbracket_{\iota} = (\bar{\ba}_1,\dots,\bar{\ba}_k),\ \llbracket(a_1,\dots,a_k)\rrbracket_{\kappa}= (\bar{\ba}_1^{\dagger},\dots,\bar{\ba}_k^{\dagger}),
\]
where homogeneous coordinates are used for $\ba_i$ and $\ba_i^{\dagger}$. Now, suppose that the change-of-basis from $\iota$ to $\kappa$ is given by $(\bB\ \bc)$ and denote
\[
\left[\bB\right]_H = \left[\begin{array}{cc}
\bB & \bc \\
\mathbf{0} & 1
\end{array}
\right]. 
\]
Then, we have
\begin{equation} \label{eq:change of affine basis product space}
\begin{aligned}
\left[\begin{array}{c}
\bar{\ba}_1^{\dagger} \\
\vdots \\
\bar{\ba}_k^{\dagger}
\end{array}\right] & = \left[\begin{array}{ccc}
\left[\bB\right]_H & &  \mathbf{0} \\
& \ddots & \\
\mathbf{0} & & \left[\bB\right]_H
\end{array}\right]\left[\begin{array}{c}
\bar{\ba}_1 \\
\vdots \\
\bar{\ba}_k
\end{array}\right] \\
& = (\mathbf{I}\otimes\left[\bB\right]_H)\left[\begin{array}{c}
\bar{\ba}_1 \\
\vdots \\
\bar{\ba}_k
\end{array}\right].
\end{aligned}
\end{equation}
Thus, the induced change-of-basis on the product space $A^K$ is given by the matrix $\mathbf{I}\otimes\left[\bB\right]_H$. 

\subsubsection{Differentials, pushforwards and pullbacks}~\mbox{}

Let $\MM$ be a smooth real manifold. For $p\in\MM$, we denote the tangent space by $T_p\MM$ and the corresponding dual space, the cotangent space by $T^*_p\MM$. Given a smooth function $f:\MM\to\mathbb{R}$, the differential $df(p)\in T^*_p\MM$ is defined by 
\[
df(p)(X_p)=X_p(f),\ X_p\in T_p\MM.
\]
Let $(\mathbf{x}_1,\dots,\mathbf{x}_n)$ be a coordinate system around $p\in\MM$, the differential $df(p)$ can be expressed as
\[
\llbracket df(p)\rrbracket = \left[\begin{array}{cc}
\frac{\partial f}{\partial\mathbf{x}_1} \\
\vdots \\
\frac{\partial f}{\partial\mathbf{x}_n}
\end{array}\right].
\]
We observe that this coordinate representation corresponds to the gradient $\nabla f$ (even though differentials and gradients are distinct objects for abstract manifolds).

For a smooth map $\varphi:\MM_1\to\MM_2$ of manifolds, the pushforward $\varphi_*:T_p\MM_1\to T_{\varphi(p)}\MM_2$ is defined by 
\[
\varphi_*(v)h = v(h\circ\varphi),
\]
where $v\in T_p\MM_1$ and $h:\MM_2\to\mathbb{R}$ is a smooth function on $\MM_2$. If we suppose that $\MM_1=\RR^n$
and $\MM_2=\RR^m$ with $(\mathbf{x}_1,\dots,\mathbf{x}_n)$ a coordinate system around $p$ and $(\mathbf{y}_1,\dots,\mathbf{y}_m)$ a coordinate system around $\varphi(p)$, the pushforward $\varphi_*v$ can be represented as
\[
\llbracket \varphi_*v \rrbracket = \left[\begin{array}{ccc}
\frac{\partial\mathbf{y}_1}{\partial\mathbf{x}_1} & \dots & \frac{\partial\mathbf{y}_1}{\partial\mathbf{x}_n} \\
\vdots & \ddots & \vdots \\
\frac{\partial\mathbf{y}_m}{\partial\mathbf{x}_1} & \dots & \frac{\partial\mathbf{y}_m}{\partial\mathbf{x}_n}
\end{array}\right]\mathbf{v},
\]
where $\mathbf{v}=\llbracket v\rrbracket$. This is exactly the Jacobian matrix $\mathbf{J}_{\varphi}$ of $\varphi$ and hence $\llbracket\varphi_*v\rrbracket = \mathbf{J}_{\varphi}\mathbf{v}$. This Jacobian-vector product corresponds to the directional derivative, and can be computed using forward mode automatic differentiation~\cite{schraudolph2002fast}.

The dual notion of the pushforward, the pullback $\varphi^*:T^*_{\varphi(p)}\MM_2\to T^*_p\MM_1$ is defined in the following way
\[
(\varphi^*u)(v)=u(\varphi_*v),\ u\in T^*_{\varphi(p)}\MM_2.
\]
With respect to the same coordinate systems chosen above, we can write $\llbracket\varphi^*u\rrbracket = \mathbf{J}^{\top}_{\varphi}\mathbf{u}$, where $\mathbf{u}=\llbracket u\rrbracket$. Numerically, we can compute $\mathbf{J}^{\top}_{\varphi}\mathbf{u}$ efficiently using reverse mode auto-differentiation (i.e., backpropagation). 

\subsubsection{Metrics and their properties}~\mbox{}

We introduce tensors on manifolds. A symmetric $k$-tensor $\sigma$ at the point $p\in\MM$ is defined as a symmetric $k$-tensor on the tangent space $T_p\MM$. Recall that this is a symmetric multilinear map on the $k$-fold product of $T_p\MM$:
\[
\sigma(p):\underbrace{T_p\MM\times\dots\times T_p\MM}_{k\ \mathrm{copies}}\to\RR.
\] 

A metric on $\MM$ is defined as a smoothly varying symmetric $2$-tensor $g$ which is positive-semidefinite at every point $p\in\MM$. Note that our definition of a metric allows the possibility of it being degenerate. If $g$ is nondegenerate, then this is just a usual Riemannian metric. However, for the remainder of this paper, we use the term \enquote{nondegenerate} rather than \enquote{Riemannian} to describe such metrics. 

In later sections, we pull back metrics from the output space to the weight space of the network. Here, we define how this works for general tensors. Let $\varphi:\MM_1\to\MM_2$ be a smooth map of manifolds and $\sigma$ be a symmetric $k$-tensor on $\MM_2$ at $\varphi(p)$. The pullback $\varphi^*\sigma$ of $\sigma$ under $\varphi$ is a symmetric $k$-tensor on $\MM_1$ defined as
\[
\varphi^*\sigma(p)(v_1,\dots,v_k) = \sigma(\varphi(p))(\varphi_*v_1,\dots,\varphi_*v_k),
\]
where $v_1,\dots,v_k\in T_p\MM_1$. In the case of metrics, 
\[
\varphi^*g_{\MM_2}(p)(v_1,v_2) = g_{\MM_2}(\varphi(p))(\varphi_*v_1,\varphi_*v_2).
\]
If we suppose that the metric $g_{\MM_2}$ is given by $\mathbf{G}_{\MM_2}$ for a chosen coordinate system around $\varphi(p)$, then the pullback metric $\varphi^*g_{\MM_2}$ on $\MM_1$ around $p$ is given by
\[ 
\llbracket\varphi^*g_{\MM_2}(p)\rrbracket = \mathbf{J}_{\varphi}^{\top}\mathbf{G}_{\MM_2}\mathbf{J}_{\varphi},
\]
where $\mathbf{J}_{\varphi}$ is the Jacobian of $\varphi$. While a metric always pulls back to a metric under a smooth map, the pullback of a nondegenerate metric can be degenerate as the pushforward map may have a non-trivial nullspace. 

\section{coordinate-free k-fac} \label{sec:abstract-mlp}

\subsection{Coordinate-free Multilayer Perceptrons} \label{subsec:coordinate-free network setup}~\mbox{}

We observe that MLPs consist of a sequence of affine transformations and activation functions in alternation. In order to capture this structure, we treat the spaces of activations and pre-activations as affine spaces. Note that this introduces more structure than was assumed when we discussed the exact natural gradient in Section~\ref{subsubsec:NG invariances}; in that section, we treated the space of network parameters as a general smooth manifold. Here, the network weights and biases are assumed to define affine transformations. The set of allowable reparameterizations (and hence, the desired set of invariances) is correspondingly more limited (though still very broad). 

We now present the coordinate-free MLP formally. For $i\in\{1,\dots,L\}$, we have
\begin{itemize}
	\item Activations are taken to be elements $\alpha_{i-1}$ in an affine space $\AA_{i-1}$.
	\item Pre-activations are taken to be elements $\zeta_i$ in an affine space $\ZZ_i$.
	\item Layerwise parameters are affine transformations $\omega_i$ between $\AA_{i-1}$ and $\ZZ_i$. The collection of these transformations is an affine space in its own right, which we denote by $\WW_i$ and refer to as the layerwise weight space.
	\item The weight space is given by the direct product $\WW=\WW_1\times\dots\times\WW_L$. Elements in this space are written as $\omega=(\omega_1,\dots,\omega_L)\in\WW$.
	\item Input and outputs are denoted by $\xi$ and $f(\xi,\omega)$ respectively. The space of all inputs and outputs are affine spaces denoted by $\XX (=\AA_0)$ and $\YY(=\AA_L)$ respectively. 
\end{itemize}
Moreover, the layerwise computation is given by
\begin{equation} \label{eq:abstract feed-forward computation}
\begin{aligned}
\zeta_i & = \omega_i(\alpha_{i-1}) \\
\alpha_i & = \rho_i(\zeta_i),
\end{aligned}
\end{equation}
where $\rho_i:\RR\to\RR$ is a fixed nonlinear activation function which is assumed to be smooth throughout. 

We highlight the power and flexibility of formulating MLPs in coordinate-free language. Suppose that the activation function $\rho_i$ is the logistic sigmoid,
\[
\rho_i(x) = \frac{e^x}{1+e^x}. 
\]
Another common activation function is $\tanh$, 
\[
\tanh(x) = \frac{e^x-e^{-x}}{e^x+e^{-x}}.
\]
An easy computation shows that $\tanh(x)=2\rho_i(2x)-1$ which means that $\tanh$ and logistic sigmoid are related to each other by an affine transformation. We can identify the pre-activation spaces for logistic and $\tanh$ networks using the isomorphism $x\mapsto 2x-1$. Similarly, we can identify the activation spaces using the isomorphism $x\mapsto\frac{1}{2}x$. Hence, the logistic and $\tanh$ architectures can be viewed as a single abstract MLP architecture with different choices of bases.

Now, a choice of parameterization, or a coordinate system, for the abstract MLP is a choice of affine bases for all of the activation spaces $\AA_1,\dots,\AA_{L-1}$, the pre-activation spaces $\ZZ_1,\dots,\ZZ_L$, the input space $\XX$ and the output space $\YY$ in the network. Observe that a choice of bases for $\AA_{i-1}$ and $\ZZ_i$ naturally induces a basis for each $\WW_i$, and therefore also for the full weight space $\WW$. Let $\iota,\kappa$ be two different choices of parameterizations for the network. With respect to $\iota$, we write
\[
\llbracket\alpha_{i-1}\rrbracket_{\iota}=\ba_{i-1},\ \llbracket\zeta_i\rrbracket_{\iota} = \bz_i,\ \llbracket\omega_i\rrbracket_{\iota} = (\bW_i\ \bb_i),\ \llbracket\rho_i\rrbracket_{\iota} = \phi_i,\ \llbracket\alpha_i\rrbracket_{\iota} = \ba_i,
\]
and with respect to $\kappa$, we write
\[
\llbracket\alpha_{i-1}\rrbracket_{\kappa}=\ba_{i-1}^{\ddagger},\ \llbracket\zeta_i\rrbracket_{\kappa} = \bz_i^{\ddagger},\ \llbracket\omega_i\rrbracket_{\kappa} = (\bW_i^{\ddagger}\ \bb_i^{\ddagger}),\ \llbracket\rho_i\rrbracket_{\kappa} = \phi_i^{\ddagger},\ \llbracket\alpha_i\rrbracket_{\kappa} = \ba_i^{\ddagger}.
\]
Hence, we can rewrite Eqn.~\ref{eq:abstract feed-forward computation} in the parameterizations $\iota,\kappa$ as
\begin{equation} \label{eq:computation in two coordinate systems}
\begin{aligned}
\bz_i & = \bW_i\ba_{i-1} + \bb_i & \bz_i^{\ddagger} & = \bW_i^{\ddagger}\ba_{i-1}^{\ddagger} + \bb_i^{\ddagger}  \\
\ba_i & = \phi_i(\bz_i), & \ba_i^{\ddagger} & = \phi_i^{\ddagger}(\bz_i^{\ddagger}).
\end{aligned}
\end{equation}
The parameters $(\bW_i\ \bb_i)$ and $(\bW_i^{\ddagger}\ \bb_i^{\ddagger})$ are related as follows
\begin{equation} \label{eq:change of basis for parameters}
\begin{aligned}
\bW_i & = \bPhi_i\bW_i^{\ddagger}\bOmega_{i-1} \\
\bb_i & = \bPhi_i\bW_i^{\ddagger}\bgamma_{i-1} + \bPhi_i\bb_i^{\ddagger} + \btau_i,
\end{aligned}
\end{equation}
where $(\bOmega_{i-1}\ \bgamma_{i-1})$ is the change-of-basis from $\iota$ to $\kappa$ on $\AA_{i-1}$ with $\bOmega_{i-1}$ an invertible matrix and $\bgamma_{i-1}$ a vector. Moreover, $(\bPhi_i\ \btau_i)$ is the change-of-basis from $\kappa$ to $\iota$ on $\ZZ_i$ with $\bPhi_i$ an invertible matrix and $\btau_i$ a vector. The activation functions $\phi_i$ and $\phi_i^{\ddagger}$ in Eqn.~\ref{eq:computation in two coordinate systems} are related in the following way:
\begin{equation} \label{eq:change of basis for activation functions}
\ba_i^{\ddagger} = \phi_i^{\ddagger}(\bz_i^{\ddagger}) = \bOmega_i\phi_i(\bPhi_i\bz_i^{\ddagger} + \btau_i) + \bgamma_i,
\end{equation}
where $(\bOmega_i\ \bgamma_i)$ is the change-of-basis from $\iota$ to $\kappa$ on $\AA_i$. The equations given in Eqn.~\ref{eq:change of basis for parameters} and Eqn.~\ref{eq:change of basis for activation functions} are standard change of variables formulas and we relegate their derivations to Appendix~\ref{Asubsec:change of affine basis computations}. These equations in Eqn.~\ref{eq:computation in two coordinate systems} can also be rewritten in homogeneous coordinates:
\begin{align*}
\bz_i & = \bar{\bW}_i\bar{\ba}_{i-1} & \bz_i^{\ddagger} & = \bar{\bW}_i\bar{\ba}_{i-1}^{\ddagger} \\
\ba_i & = \phi_i(\bz_i), & \ba_i^{\ddagger} & = \phi_i^{\ddagger}(\bz_i^{\ddagger}) \\
& & & = \bOmega_i\phi_i(\bPhi_i\bz_i^{\ddagger} + \btau_i) + \bgamma_i.
\end{align*}

The left hand set of equations above is identical to the original MLP computation given in Eqn.~\ref{eq:homogeneous feed-forward computation}. The right hand set of equations is identical to the transformed computation given in Eqn.~\ref{eq:transformed network computation}. Thus, we arrive at a very important point here: the MLP with computation defined in Eqn.~\ref{eq:homogeneous feed-forward computation} and the transformed version in Eqn.~\ref{eq:transformed network computation} simply correspond to two different choices of parameterizations for the same underlying abstract MLP. 

\subsection{Optimization problem for abstract networks.}~\mbox{}

The optimization problem in the abstract setting is analogous to the coordinate-dependent one. Let $(\xi,\upsilon)$ be an abstract input-target pair and $\LL(\upsilon,f(\xi,\omega))$ be the loss function measuring the disagreement between outputs $f(\xi,\omega)$ of the abstract MLP and targets $\upsilon$. Given a training set $\mathcal{S}$ of abstract input-target pairs $(\xi_i,\upsilon_i)$, the objective function we wish to minimize here is
\[
h(\omega) = \frac{1}{|\mathcal{S}|}\sum_{(\xi_i,\upsilon_i)\in\mathcal{S}}\LL(\upsilon_i,f(\xi_i,\omega)).
\]

\begin{Theorem} \label{NG invariances theorem}
Let $g$ be a nondegenerate metric on the weight space $\WW$ of an abstract MLP. For a chosen learning rate $\epsilon > 0$, the following update rule
\[
\omega\leftarrow\omega-\epsilon g(\omega)^{-1}dh(\omega),
\]
is~\emph{exactly} invariant to all affine reparameterizations of the model.
\end{Theorem}

\begin{proof}
First, note that $g(\omega)^{-1}dh(\omega)$ is an intrinsically defined tangent vector on $\WW$. The weight space $\WW$ of a MLP is an affine space, and hence by Corollary~\ref{tangent spaces of affine spaces corollary} in Appendix, the tangent space of $\WW$ at every point is canonically isomorphic to the vector space naturally associated to $\WW$. Thus, the exponential map $\Exp_{\omega}$ corresponds to the above update rule. Since this construction did not require choosing an affine basis for $\WW$, the algorithm is invariant to affine reparameterizations.
\end{proof}

We describe the consequences of this theorem more concretely using the parameterizations $\iota,\kappa$ given earlier in Section~\ref{subsec:coordinate-free network setup}. Suppose that 
\[
\llbracket\omega\rrbracket_{\iota} = \bw,\ \llbracket\omega\rrbracket_{\kappa} = \bw^{\ddagger},\ \llbracket g(\omega)^{-1}\rrbracket_{\iota} = \bG(\bw)^{-1},\ \llbracket g(\omega)^{-1}\rrbracket_{\kappa} = \bG^{\ddagger}(\bw^{\ddagger})^{-1}.
\]
Furthermore, 
\[
\llbracket g(\omega)^{-1}dh(\omega)\rrbracket_{\iota} = \llbracket g(\omega)^{-1}\rrbracket_{\iota}\llbracket dh(\omega)\rrbracket_{\iota} = \bG(\bw)^{-1}\nabla h(\bw),
\]
and analogously for $\kappa$,
\[
\llbracket g(\omega)^{-1}dh(\omega)\rrbracket_{\kappa} = \llbracket g(\omega)^{-1}\rrbracket_{\kappa}\llbracket dh(\omega)\rrbracket_{\kappa} = \bG^{\ddagger}(\bw^{\ddagger})^{-1}\nabla h(\bw^{\ddagger}).
\]
The above theorem shows that the update $\bw\leftarrow\bw-\epsilon\bG(\bw)^{-1}\nabla h(\bw)$ is equivalent to the update $\bw^{\ddagger}\leftarrow\bw^{\ddagger}-\epsilon\bG^{\ddagger}(\bw^{\ddagger})^{-1}\nabla h(\bw^{\ddagger})$, in that the functions computed by the resulting networks are identical. Note that the resulting networks are \emph{exactly} equivalent, in contrast to using the natural gradient in Eqn.~\ref{eq:NGD update rule}, where the equivalence only held up to the first-order as explained at the end of Section~\ref{subsubsec:NG invariances}. Also, observe that this result holds for arbitrary metrics, not just the Fisher metric; we'll make use of this when we analyze the K-FAC metric.

\subsection{Pullback of output metrics to parameter spaces}~\mbox{} \label{subsec:pullback of metrics}

Consider a metric $g$ on the output space $\YY$ of the MLP. Let $\Psi_{\xi}:\WW\to\YY$ be the smooth map which sends parameters $\omega$ to outputs $\alpha_L=f(\xi,\omega)$ given an input $\xi$. The pullback $\Psi_{\xi}^*g$ defines a metric on $\WW$. The expected pullback metric over inputs, under a choice of coordinates around $\omega$ and $\alpha_L$, is given by
\[
\llbracket\EE_{\xi}[\Psi_{\xi}^*g(\alpha_L)]\rrbracket = \EE_{\bx}[\bJ_{\Psi_{\xi}}^{\top}\bG\bJ_{\Psi_{\xi}}],
\]
where $\bG$ is the representation of $g$  in these coordinates. We now observe how this construction encompasses a variety of examples.

\begin{Example}[Fisher metric] \label{Fisher metric example}
Suppose that the outputs $\alpha_L$ parameterize the model's predictive distribution $R_{\upsilon|\alpha_L}$. Let $r(\upsilon|\alpha_L)$ denote the density function of this distribution and furthermore, we take the loss function here to be the negative log-likelihood $\LL_{\alpha_L}=-\log r(\upsilon|\alpha_L)$. The output Fisher metric $g_{F,\mathrm{out}}$ on $\YY$ is defined as
\[
g_{F,\mathrm{out}}(\alpha_L) = \EE_{\upsilon}[d\LL_{\alpha_L}\otimes d\LL_{\alpha_L}],
\]
where the expectation is taken with respect to the predictive distribution $R_{\upsilon|\alpha_L}$. Computing the expectation of $\Psi_{\xi}^*g_{F,\mathrm{out}}$ over the inputs $\xi$ gives
\begin{align*}
\EE_{\xi}[\Psi_{\xi}^*g_{F,\mathrm{out}}(\omega)] & = \EE_{\xi}[\Psi_{\xi}^* \EE_{\upsilon}[d\LL_{\alpha_L}\otimes d\LL_{\alpha_L}]] \\
& = \EE_{\xi,\upsilon}[\Psi_{\xi}^*(d\LL_{\alpha_L}\otimes d\LL_{\alpha_L})] \\
& = \EE_{\xi,\upsilon}[\Psi_{\xi}^*(d\LL_{\alpha_L})\otimes\Psi_{\xi}^*(d\LL_{\alpha_L})] \\
& = \EE_{\xi,\upsilon}[d\LL_{\omega}\otimes d\LL_{\omega}].
\end{align*}
This is exactly the Fisher metric defined earlier in Eqn.~\ref{eq:Fisher metric on parameter space}.
\end{Example}

\begin{Example}[Gauss-Newton] \label{Gauss-Newton metric example}
Let $g_E$ be the Euclidean metric on $\YY$. Upon a choice of coordinate system, $g_E$ can be represented by the identity matrix. Then, the pullback $\Psi_{\xi}^*g_E(\omega)$ is
\[
\llbracket\Psi_{\xi}^*g_E(\omega)\rrbracket = \bJ^{\top}_{\Psi_{\xi}}\bJ_{\Psi_{\xi}}.
\] 
Now, the expectation of $\Psi_{\xi}^*g_E$ over inputs $\xi$ in these chosen coordinates is 
\begin{align*}
\llbracket\EE_{\xi}[\Psi_{\xi}^*g_E(\omega)] \rrbracket & = \EE_{\bx}[\llbracket\Psi_{\xi}^*g_E(\omega)\rrbracket] \\
& = \EE_{\bx}[\bJ^{\top}_{\Psi_{\xi}}\bJ_{\Psi_{\xi}}].
\end{align*}
We note that this is exactly the standard Gauss-Newton matrix \cite{Martens14}. One use case for the Gauss-Newton metric is when the outputs of the network do not have a natural probabilistic interpretation, e.g.~the value network in an actor-critic architecture for reinforcement learning \cite{wu2017scalable}.
\end{Example}

\begin{Example}[Generalized Gauss-Newton]
Let $F:\YY\to\RR$ be a strictly convex twice-differentiable function. The Bregman divergence $B_F:\YY\times\YY\to\RR^{+}$ is defined as
\[
B_F(y,y') = F(y) - F(y') - \langle\nabla F(y), y-y'\rangle. 
\]
The second-order Taylor approximation of this divergence is given by the Hessian of $F$, $\mathbf{H}_F = \nabla^2F$. This defines a metric on $\YY$. The pullback to $\WW$ in coordinates is the matrix $\bJ^{\top}_{\Psi_{\xi}}\mathbf{H}_F\bJ_{\Psi_{\xi}}$. Then, taking the expectation over inputs yields
\[
\EE_{\bx}[\bJ^{\top}_{\Psi_{\xi}}\mathbf{H}_F\bJ_{\Psi_{\xi}}],
\]
which is exactly the generalized Gauss-Newton matrix \cite{Martens14}.
\end{Example}

\subsection{Independence metric} \label{subsec:independence metric}~\mbox{}

We now come to the heart of our paper: the construction of a metric inspired by the K-FAC approximation. Recall that K-FAC makes two approximations to obtain a tractable Fisher matrix: (1) it assumes independence of activations and pre-activation derivatives in order to push the expectation inside the Kronecker product (Eqn.~\ref{eq:KFAC decomposition blocks}), and (2) it keeps only the diagonal blocks corresponding to individual layers. In this section, we develop a coordinate-free way to push the expectation inside the Kronecker product, thereby obtaining an approximate metric we term the \emph{independence metric}. (We later use this construction to develop approximate metrics for MLPs, convolutional networks, and RNNs.) In Section~\ref{subsec:KFAC metric}, we develop a coordinate-free version of the block-diagonal approximation. Combining both approximations yields the \emph{K-FAC metric}, an intrinsic metric whose coordinate representation matches the K-FAC approximate Fisher matrix.

We begin by setting up the mathematical framework. To avoid tying ourselves to MLPs, we consider the more general setting of metrics on affine maps between affine spaces, but use notation which is suggestive of MLPs. We assume the following:
\begin{itemize}
	\item Affine spaces $A$ and $Z$
	\item Affine space $W$ of affine transformations between $A$ and $Z$
	\item Metric $g$ on $Z$
\end{itemize}

Our first task is to formulate a coordinate-free analogue of the outer product of homogenized activations, $\bar{\mathbf{a}}_i \bar{\mathbf{a}}_i^\top$. Consider the evaluation map $\psi_a:W\to Z$ which is defined by evaluating $w$ at $a$. We compute the pushforward $\psi_{a*}:TW\to TZ$. Note that there is no need to specify particular points for the tangent spaces here since we are working with affine spaces (see Corollary~\ref{tangent spaces of affine spaces corollary} in Appendix). Let $\partial_w$ be a tangent vector on $W$ and $f$ be a smooth function on $Z$. Then,
\begin{align*}
\psi_{a*}(\partial_w)(f) & = \partial_w(f\circ\psi_a)(w) \\
& = (\partial_w f)(\psi_a(w))\cdot\psi_a'(w) \\
& = (\partial_w f)(z)\cdot a
\end{align*}
This shows that the pushforward $\psi_{a*}$ is exactly multiplication by the element $a$. Hence, we can identify any $a\in A$ with its linear map $TW\to TZ$. Thus, this enables us to define the tensor product of two elements in $A$ as a mapping $a_1 \otimes a_2 : TW\times TW\to TZ\otimes TZ$:
\[
(a_1 \otimes a_2)(\partial_{w_1},\partial_{w_2}) = a_1(\partial_{w_1}) \otimes a_2(\partial_{w_2}). 
\]
We now introduce the central object of our study, inspired by the independence assumption for activations and pre-activation derivatives which led to Eqn.~\ref{eq:KFAC decomposition blocks}.  For $w\in W$, define $g_{\mathrm{ind}}$ on $W$ to be
\begin{equation} \label{eq:abstract KFAC metric}
g_{\mathrm{ind}}(w) = \EE[a\otimes a]\otimes \EE[g(z)],
\end{equation}
where the first expectation is over $A$ and the second one is over $Z$. Note that $\EE[g(z)]$ is well defined because the affine structure of $Z$ allows us to identify the cotangent spaces at all points $z$. Our goal is to show that $g_{\mathrm{ind}}$ is a metric on $W$. Before doing so, we  establish the following lemma:

\begin{Lemma} \label{general KFAC metric lemma}
Let $g$ be a metric on $Z$ and $\psi_a: W\to Z$ be the evaluation map. Then, the pullback metric $\psi_a^*g$ on $W$ can be expressed as:
\[
\psi_a^*g(w) = \EE[a\otimes\phi\otimes a\otimes\phi],
\]
where $\phi$ is a random covector. 
\end{Lemma}

\begin{proof}
Given $z\in Z$, the metric $g$ admits the rank-1 decomposition
\[
g(z)=\EE[\phi\otimes\phi],
\]
where $\phi$ is a covector and the expectation is over $Z$. This is akin to the more familiar case where any symmetric positive-semidefinite matrix admits a rank-1 spectral decomposition. Computing the pullback of $g$ under the map $\psi_a$ now gives
\begin{equation} \label{eq:pullback of metric g on Z}
\begin{aligned}
\psi_a^*g(w) & = \psi_a^*\EE[\phi\otimes\phi] \\
& = \EE[\psi_a^*(\phi\otimes\phi)] \\
& = \EE[\psi_a^*\phi\otimes\psi_a^*\phi],
\end{aligned}
\end{equation}
where $z=\psi_a(w)$. We analyze the pullback $\psi_a^*\phi$. Let $\partial_w$ be a tangent vector on $W$. Then,
\begin{align*}
(\psi_a^*\phi)(\partial_w) & = \phi((\psi_{a*})\partial_w) \\
& = \phi(\partial_w\cdot a) \\
& = (a\otimes\phi)(\partial_w),
\end{align*}
which shows that $\psi_a^*\phi = a\otimes\phi$. Plugging this back into Eqn.~\ref{eq:pullback of metric g on Z}, we obtain
\[
\psi_a^*g(w) = \EE[a\otimes\phi\otimes a\otimes\phi],
\]
which concludes the proof.
\end{proof}

\begin{Theorem} \label{general KFAC theorem}
Let $g$ be a metric on $Z$ and $\psi_a:W\to Z$ be the evaluation map. Then, $g_{\mathrm{ind}}$ as defined in Eqn.~\ref{eq:abstract KFAC metric} is a metric on $W$. Moreover, if the expected pullback metric $\EE_a[\psi_a^*g]$ is nondegenerate on $W$, then $g_{\mathrm{ind}}$ is also nondegenerate. From now on, we refer to $g_{\mathrm{ind}}$ as the~\emph{independence metric}.
\end{Theorem}

\begin{proof}
For the first assertion, we need to check that both components $\EE[a\otimes a]$ and $\EE[g(z)]$ define symmetric positive-semidefinite 2-tensors. Recall that $a$ can be realized as a linear map from $TW$ to $TZ$. Then, the dual element $\lambda$ is a map from $T^*W$ to $T^*Z$. We defer to Appendix~\ref{Asubsec: tensor product of VS} for a formal explanation of this. To check the positive-semidefinite property, 
\[
\EE[a\otimes a](\lambda,\lambda) = \EE[(a\otimes a)(\lambda\otimes\lambda)] \geq 0,
\]
where the latter inequality is due to the fact that $a\otimes a$ is positive-semidefinite. Moreover, $a\otimes a$ is also symmetric and this property is preserved under expectations which implies that $\EE[a\otimes a]$ is both symmetric and positive-semidefinite. For the second term $\EE[g(z)]$ in $g_{\mathrm{ind}}(w)$, the fact that $g$ is a metric on $Z$ means that $g(z)$, by definition, is a symmetric positive-semidefinite 2-tensor on $T^*Z$.

To establish the second assertion of the theorem, we need to show that both $\EE[a\otimes a]$ and $\EE[g(z)]$ are positive-definite. Suppose to the contrary that this is not true for $\EE[a\otimes a]$. Then there exists $\lambda\in TW\otimes T^*Z$ (this is same as saying $\lambda$ is a linear map from $T^*W$ to $T^*Z$; refer to Appendix~\ref{Asubsec: tensor product of VS} for further explanations) such that
\begin{align*}
0 & = \EE[a \otimes a](\lambda,\lambda)  \\
& = \EE[(a\otimes a)(\lambda\otimes\lambda)],
\end{align*}
and hence $(a\otimes a)(\lambda\otimes\lambda)=0$. Now, consider the element $\lambda\otimes\mu\in (TW\otimes T^*Z)\otimes TZ$ where $\mu\in TZ$. We evaluate $\EE_a[\psi_a^*g(w)]$ at $\lambda\otimes\mu$:
\begin{align*}
\EE_a[\psi_a^*g(w)](\lambda\otimes\mu,\lambda\otimes\mu) & = \EE[a\otimes\phi\otimes a\otimes\phi](\lambda\otimes\mu\otimes\lambda\otimes\mu) \\
& = \EE[\underbrace{(a \otimes a)(\lambda\otimes\lambda)}_{=0}\cdot(\phi\otimes\phi)(\mu\otimes\mu)] \\
& = 0,
\end{align*}
where we use the result of Lemma~\ref{general KFAC metric lemma} in the first equality. This shows that $\EE_a[\psi_a^*g]$ is not positive-definite which yields a contradiction as $\EE_a[\psi_a^*g]$ was assumed to be a nondegenerate metric. The exact same argument can be applied to show that $\EE[g(z)]$ is positive-definite. This gives us the desired result.
\end{proof}

We finally show that the coordinate representation of $g_{\mathrm{ind}}$ matches the K-FAC approximation to the layerwise Fisher blocks.

\begin{Proposition} \label{general K-FAC metric coordinates prop}
Suppose that we choose coordinate systems for the affine spaces $A$, $Z$ and in these coordinates, 
\[
\llbracket a\rrbracket = \ba,\ \llbracket z\rrbracket = \bz,\ \llbracket g(z)\rrbracket = \bG(\bz).
\]
Then the independence metric $g_{\mathrm{ind}}$ can be expressed as
\[
\llbracket g_{\mathrm{ind}}(w)\rrbracket = \EE[\bar{\ba}\bar{\ba}^{\top}]\otimes\EE[\bG(\bz)].
\]
\end{Proposition}

\begin{proof}
This is by direct computation
\begin{align*}
\llbracket g_{\mathrm{KFAC}}(w)\rrbracket & = \llbracket \EE[\alpha\otimes\alpha]\otimes\EE[g(z)] \rrbracket \\
& = \EE[\llbracket\alpha\otimes\alpha\rrbracket] \otimes\EE[\llbracket g(z)\rrbracket] \\
& = \EE[\bar{\ba}\bar{\ba}^{\top}]\otimes\EE[\bG(\bz)],
\end{align*}
where we use the homogeneous notation $\bar{\ba}$ in the last equality.
\end{proof}

\begin{Remark}
In the context of MLPs (which we explain in much greater detail subsequently) where $A=\AA_{i-1}$, $Z=\ZZ_i$ and $W=\WW_i$, the matrix $\bG(z)=\EE_{\by}[\DD\bz_i\DD\bz_i^{\top}]$ where the expectation is taken over output space. $\EE[\bG(z)]$ means we furthermore take the expectation over $\ZZ_i$. When we write $\EE[\DD\bz_i\DD\bz_i^{\top}]$ in Eqn.~\ref{eq:KFAC decomposition blocks}, we implicitly take this to mean $\EE[\bG(z)]$. 
\end{Remark}

\subsection{K-FAC metric}~\mbox{} \label{subsec:KFAC metric}

In this section, we formulate the layerwise independence approximation in a coordinate-free way, allowing us to define the K-FAC metric, whose coordinate representation matches the K-FAC approximation to the Fisher matrix. 
We begin by introducing the notion of additive metrics on product manifolds. Next, we proceed to use the independence metric developed in Section~\ref{subsec:independence metric} to define the K-FAC metric for MLPs. Lastly, by viewing K-FAC as a metric on $\WW$, we show how invariances of the K-FAC algorithm can be obtained in a very natural and straightforward manner.

\subsubsection{Additive metrics}~\mbox{}

Given metrics $g_{\MM_1}$ and $g_{\MM_2}$ on $\MM_1$ and $\MM_2$ respectively, we describe how to naturally define a metric on the product manifold $\MM_1\times\MM_2$. For any point $(p,q)\in\MM_1\times\MM_2$, there is a canonical isomorphism of tangent spaces:
\[
T_{(p,q)}(\MM_1\times\MM_2)\cong T_p\MM_1\oplus T_q\MM_2.
\]
The proof of this fact can be found in standard differential geometry literature~\cite{lee2003smooth} and so we do not elaborate further here. Hence, any vector $v\in T_{(p,q)}(\MM_1\times\MM_2)$ can be written as a pair $(v_1,v_2)$ where $v_1\in T_p\MM_1$ and $v_2\in T_q\MM_2$. Now, we define the additive metric $g_{\MM_1}+g_{\MM_2}$ on $\MM_1\times\MM_2$ as follows:
\begin{equation} \label{eq:product metric on tangent vectors}
(g_{\MM_1}+g_{\MM_2})(p,q)(u,v) = g_{\MM_1}(p)(u_1,v_1) + g_{\MM_2}(q)(u_2,v_2). 
\end{equation}
If we choose a coordinate system around $(p,q)$ with the metrics $g_{\MM_1}$, $g_{\MM_2}$ represented by matrices $\mathbf{G}_{\MM_1}$, $\mathbf{G}_{\MM_2}$ respectively, then we have
\[
\llbracket (g_{\MM_1} + g_{\MM_2})(p,q)\rrbracket = \left[\begin{array}{cc}
\mathbf{G}_{\MM_1} & \mathbf{0} \\
\mathbf{0} & \mathbf{G}_{\MM_2} 
\end{array}\right],
\]
which is a matrix with block diagonals $\mathbf{G}_{\MM_1},\mathbf{G}_{\MM_2}$ and zero everywhere else. This construction generalizes easily to sums of more than two terms.

\subsubsection{Coordinate-free K-FAC for MLPs}~\mbox{}

Consider a MLP with $L$ layers as described earlier in Section~\ref{subsec:coordinate-free network setup}. For every $i\in\{1,\dots,L\}$, we define the following maps 
\begin{itemize}
	\item $\psi_{\xi}^i:\WW_i\to\ZZ_i$ which sends layerwise parameters $\omega_i$ to pre-activations $\zeta_i$ by evaluation at activations $\alpha_{i-1}$. 
	\item $\varphi_{\xi}^i:\ZZ_i\to\YY$ which sends $\zeta_i$ to network outputs $\alpha_L=f(\xi,\omega)$.
\end{itemize}
Note that $\psi_{\xi}^i$ is a smooth map by definition. Now, observe that $\varphi_{\xi}^i$ is exactly the composition of network maps
\[
\rho_L\circ\omega_L\circ\dots\circ\omega_{i+1}\circ\rho_i:\ZZ_i\to\YY.
\]
Since all activation functions $\rho_i$ are assumed to be smooth maps, it follows immediately that $\varphi_{\xi}^i$ is also a smooth map.  Moreover, consider the map $\Psi_{\xi}^i:\WW_i\to\YY$ which is defined as the composition $\Psi_{\xi}^i = \varphi_{\xi}^i\circ\psi_{\xi}^i$. The subscript $\xi$ is used to highlight the fact that all of these maps implicitly depend on network inputs $\xi$.

Let $g$ be a metric on $\YY$. Then, the pullback $(\varphi_{\xi}^i)^*g$ defines a metric on $\ZZ_i$. Now, if we take $A$, $Z$, $W$ in Section~\ref{subsec:independence metric} to be
\[
A=\AA_{i-1},\ Z=\ZZ_i,\ W=\WW_i,
\]
and the metric on $Z=\ZZ_i$ to be $(\varphi_{\xi}^i)^*g$, the independence metric on $\WW_i$ here is 
\begin{equation} \label{eq:abstract KFAC metric MLP components}
g_{\mathrm{ind}}^i(\omega_i) = \EE[\alpha_{i-1}\otimes\alpha_{i-1}]\otimes\EE[(\varphi_{\xi}^i)^*g(\zeta_i)].
\end{equation}

\begin{Definition} \label{abstract K-FAC metric MLP definition}
The K-FAC metric on the weight space $\WW$ of a MLP is defined as
\[
g_{\mathrm{KFAC}}(\omega) = g_{\mathrm{ind}}^1(\omega_1) + \dots + g_{\mathrm{ind}}^L(\omega_L),
\]
where the sum above is as defined in Eqn.~\ref{eq:product metric on tangent vectors} and each $g_{\mathrm{ind}}^i$ is as given in Eqn.~\ref{eq:abstract KFAC metric MLP components}.
\end{Definition}

\begin{Theorem} \label{abstract KFAC theorem for MLP}
Let $g$ be a metric on $\YY$. Then, $g_{\mathrm{KFAC}}$ given in Definition~\ref{abstract K-FAC metric MLP definition} is indeed a metric on the weight space $\WW$ of an abstract MLP. Moreover, if we assume that the expected pullback of $g$,
\[
\EE_{\xi}[(\Psi_{\xi}^i)^*g], 
\]
under the map $\Psi_{\xi}^i:\WW_i\to\YY$ is a nondegenerate metric on the layerwise weight space $\WW_i$ for every $i$, then $g_{\mathrm{KFAC}}$ is also nondegenerate. 
\end{Theorem}

\begin{proof}
From Theorem~\ref{general KFAC theorem}, we know that $g_{\mathrm{ind}}^i$ is a metric on $\WW_i$. Since $g_{\mathrm{KFAC}}$ is defined as the additive metric where each of the summands are $g_{\mathrm{ind}}^i$, we can conclude that  $g_{\mathrm{KFAC}}$ is a metric. For the second assertion of the theorem, recall that $\Psi_{\xi}^i = \varphi_{\xi}^i\circ\psi_{\xi}^i$. By the functoriality property of pullback operations, we have
\[
(\Psi_{\xi}^i)^* = (\psi_{\xi}^i)^*\circ(\varphi_{\xi}^i)^*.
\]
Since $\EE_{\xi}[(\Psi_{\xi}^i)^*g]$ was assumed to be nondegenerate, this implies that
\[
\EE_{\xi}[(\psi_{\xi}^i)^*((\varphi_{\xi}^i)^*g)],
\]
is also nondegenerate. Then, from the second assertion of Theorem~\ref{general KFAC theorem}, we obtain that $g_{\mathrm{ind}}^i$ is a nondegenerate metric on $\WW_i$. Consequently, $g_{\mathrm{KFAC}}$ is nondegenerate which concludes the proof.
\end{proof}

\begin{Remark}
We like to remark here that we can fit the K-FAC metric for a metric $g$ on the output space $\YY$ which is not the output Fisher metric $g_{F,\mathrm{out}}$ in Example~\ref{Fisher metric example}. To do so, we sample a covector $\phi$ on $\YY$ whose expected tensor product is $\EE[\phi\otimes\phi]=g$. Then, taking expectation of the tensor product of pullback samples is the pullback of $g$.
\end{Remark}

\noindent {\bf {Coordinate-free proof of Theorem~\ref{KFAC invariances theorem}}.} We can now provide a natural and straightforward proof of Theorem~\ref{KFAC invariances theorem}. We have already shown in Section~\ref{subsec:coordinate-free network setup} that the networks $\NN$ and $\NN^{\dagger}$ correspond to two different choices of parameterizations for the same underlying abstract MLP. Hence, they must compute the same function. 

Assume that the metric $g$ on the output space $\YY$ in Theorem~\ref{abstract KFAC theorem for MLP} is the output Fisher metric $g_{F,\mathrm{out}}$ in Example~\ref{Fisher metric example}. The pullback of this under $\varphi_{\xi}^i$ is given by
\[
(\varphi_{\xi}^i)^*g_{F,\mathrm{out}}(\zeta_i) = \EE[d\LL_{\zeta_i}\otimes d\LL_{\zeta_i}].
\]
Let us choose coordinate systems on $\AA_{i-1}$ and $\ZZ_i$ with
\[
\llbracket \alpha_{i-1} \rrbracket =\ba_{i-1},\ \llbracket \zeta_i\rrbracket = \bz_i,\ \llbracket(\varphi_{\xi}^i)^*g_{F,\mathrm{out}}(\zeta_i)\rrbracket = \llbracket \EE[d\LL_{\zeta_i}\otimes d\LL_{\zeta_i}] \rrbracket = \EE[\DD\bz_i\DD\bz_i^{\top}].
\]
Then, by Proposition~\ref{general K-FAC metric coordinates prop},
\[
\llbracket g_{\mathrm{ind}}^i(\omega_i)\rrbracket = \EE[\bar{\ba}_{i-1}\bar{\ba}_{i-1}^{\top}]\otimes\EE[\DD\bz_i\DD\bz_i^{\top}],
\]
which is exactly $\hat{\bF}(\bw)_{i,i}$ given earlier in Eqn.~\ref{eq:KFAC decomposition blocks}. Furthermore, 
\[
\llbracket g_{\mathrm{KFAC}}(\omega)\rrbracket = \llbracket g_{\mathrm{ind}}^1(\omega_1) + \dots + g_{\mathrm{ind}}^L(\omega_L)\rrbracket
\]
is the matrix with diagonal blocks $\hat{\bF}(\bw)_{i,i}$ and zeros everywhere else. This is precisely $\hat{\bF}(\bw)$ in Eqn.~\ref{eq:KFAC matrix}. Now, observe that
\[
\llbracket g_{\mathrm{KFAC}}(\omega)^{-1}dh(\omega)\rrbracket = \llbracket g_{\mathrm{KFAC}}(\omega)^{-1}\rrbracket\llbracket dh(\omega)\rrbracket = \hat{\bF}(\bw)^{-1}\nabla h(\bw),
\]
and hence the K-FAC update rule in Eqn.~\ref{eq:KFAC update rule} is simply a natural gradient update rule with respect to the K-FAC metric $g_{\mathrm{KFAC}}$ for abstract MLPs. Suppose that $g_{\mathrm{KFAC}}$ is a nondegenerate metric; which is true for example if the assumptions in the second assertion of Theorem~\ref{abstract KFAC theorem for MLP} hold. Applying Theorem~\ref{NG invariances theorem} shows that this update rule is invariant to any affine reparameterizations of the model. \qed

\begin{Remark}
Note that in the proof above we made an assumption that the K-FAC metric $g_{\mathrm{KFAC}}$ is nondegenerate. In order to handle such degeneracies occurring in practical situations, one often adds a damping term $\gamma\mathbf{I}$ to the K-FAC approximation $\hat{\bF}(\bw)$. The invariance properties of the K-FAC update rules are no longer preserved if we add this damping term; however, if the effect of it is small, then the update is approximately invariant. We refer to~\cite{KFAC15} for more extensive details on effective damping techniques. 
\end{Remark}

\section{coordinate-free k-fac for convolutional networks} \label{sec:abstract-CNN}

In the remaining two sections, we extend the preceding analysis to convolutional networks and recurrent neural networks. Both cases are straightforward applications of the results from Section \ref{sec:abstract-mlp}, highlighting the flexibility of our analysis.

\subsection{Convolutional networks}~\mbox{}

We begin by describing the convolution layer of a convolutional network in mathematical terms following~\cite{KFC}. It suffices to only consider convolution layers as the pooling and response normalization layers of a convolutional network typically do not contain (many) trainable weights. We then introduce the notion of a transformed convolution layer analogous to what was done in the case of MLPs. Lastly, we use the abstract linear algebra machinery developed in Section~\ref{sec:background} to give a coordinate-free description of convolution layers.

\subsubsection{Convolution layers}~\mbox{} \label{subsubsec:convolution layers}

We focus on a single convolution layer. A convolution layer $l$ takes as input a layer of activations $\ba_{j,t}$, where $j\in\{1,\dots,J\}$ indexes the input map and $t\in\TT$ indexes the spatial location. $\TT$ here denotes the set of spatial locations, which we typically take to be a 2D-grid. We assume that the convolution is performed with a stride of 1 and padding equal to the kernel radius $R$, so that the set of spatial locations is shared between the input and output feature maps. This layer is parameterized by a set of weights $\bw_{i,j,\delta}$ and biases $\bb_i$, where $i\in\{1,\dots,I\}$ indexes the output map and $\delta\in\Delta$ indexes the spatial offset. The numbers of spatial locations and spatial offsets are denoted by $|\TT|$ and $|\Delta|$ respectively. The computation at the convolution layer is given by
\begin{equation} \label{eq:convolution layer computation}
\bz_{i,t} = \sum_{\delta\in\Delta}\bw_{i,j,\delta}\ba_{j,t+\delta} + \bb_i.
\end{equation}
The pre-activations $\bz_{i,t}$ are then passed through a nonlinear activation function $\phi_l$. Analogous to feed-forward networks, the weight derivatives are computed using backpropagation:
\[
\DD\bw_{i,j,\delta} = \sum_{t\in\TT}\ba_{j,t+\delta}\DD\bz_{i,t}. 
\]

Following \cite{KFC}, we represent the convolution layer computation in Eqn.~\ref{eq:convolution layer computation} using matrix notation. To do this, we write the activations $\ba_{j,t}$ as a $J\times|\TT|$ matrix $\bA_{l-1}$, the pre-activations $\bz_{i,t}$ as a $I\times|\TT|$ matrix $\bZ_l$, the weights $\bw_{i,j,\delta}$ as a $I\times J|\Delta|$ matrix $\bW_l$ and the bias vector as $\bb_l$. For the activation matrix $\bA_{l-1}$, if we extract the patches surrounding each spatial location $t\in\TT$ and flatten these patches into vectors where the vectors become columns of a matrix, we obtain a $J|\Delta|\times|\TT|$ matrix which we denote by $\bA_{l-1}^{\exp}$. From now on, we refer to this matrix as the expanded activations. Finally, we can use these matrix notations to rewrite the computation in Eqn.~\ref{eq:convolution layer computation} as
\begin{equation} \label{eq:convolution layer matrix multiplication}
\begin{aligned}
\bZ_l & = \bW_l\bA_{l-1}^{\exp} + \bb_l \\
\bA_l & = \phi_l(\bZ_l).
\end{aligned}
\end{equation}
For convenience purposes later, we adopt homogeneous coordinates for various matrices:
\[
[\bA_{l-1}^{\exp}]_H = \left[\begin{array}{c}
\bA_{l-1}^{\exp} \\
\mathbf{1}
\end{array}\right],
[\bZ_{l-1}]_H = \left[\begin{array}{c}
\bZ_l \\
\mathbf{1}
\end{array}\right],
[\bW_l]_H = \left[\begin{array}{cc}
\bW_l & \bb_l \\
\mathbf{0} & 1
\end{array}\right],
[\bA_l]_H = \left[\begin{array}{c}
\bA_l \\
\mathbf{1}
\end{array}\right].
\]
Hence, Eqn.~\ref{eq:convolution layer matrix multiplication} can be rewritten as
\begin{equation} \label{eq:homogeneous convolution layer computation}
\begin{aligned}
[\bZ_l]_H & = [\bW_l]_H[\bA_{l-1}^{\exp}]_H \\
[\bA_l]_H & = \phi_l([\bZ_l]_H), 
\end{aligned}
\end{equation}
where the activation function $\phi_l$ here ignores the homogeneous coordinate. 

We briefly introduce the concept of a transformed convolution layer. For a convolution layer as defined in Eqn.~\ref{eq:homogeneous convolution layer computation}, the parameters $[\bW_l]_H$ and the transformed parameters $[\bW_l^{\dagger}]_H$ are related in the following way
\begin{equation} \label{eq:transformed convolution layer parameters}
[\bW_l]_H = \bm{\Gamma}_l[\bW_l^{\dagger}]_H(\bI\otimes\bm{\Upsilon}_{l-1}),
\end{equation}
where $\bm{\Gamma}_l$ and $\bm{\Upsilon}_l$ are invertible matrices. The activation functions $\phi_l$ and $\phi_l^{\dagger}$ are related through a standard affine change-of-basis as given in Eqn.~\ref{eq:transformed network computation}.

\subsubsection{Abstract convolution layers}~\mbox{} \label{subsubsec:abstract conv layer}

Just as in the coordinate-dependent case earlier, we focus on a single layer. An abstract convolution layer $l$ is defined as follows:
\begin{itemize}
	\item Local activations at each spatial location $t\in\TT$ are taken to be elements $\alpha_{l-1}$ in an affine space $\AA_{l-1}$.
	\item Activations are taken to be elements  $\alpha_{l-1}^{(:)}$ in $\AA_{l-1}^{|\TT|}$, (i.e.~the direct product of $\AA$, $|\TT|$ times). (The superscripts are meant to be suggestive of Python slicing notation.)
	\item Expanded activations at $t\in\TT$ are taken to be elements $\alpha_{l-1}^{(:,t)}$ in $\AA_{l-1}^{|\Delta|}$. The full expanded activations are taken to be elements $\alpha_{l-1}^{(:,:)}$ in $\AA_{l-1}^{|\Delta|\otimes|\TT|}$.
	\item Local pre-activations at $t\in\TT$ are taken to be elements $\zeta_l^{(t)}$ in an affine space $\ZZ_l$.
	\item Pre-activations are taken to be elements $\zeta_l^{(:)}$ in $\ZZ_l^{|\TT|}$.
	\item Layerwise parameters are affine transformations $\omega_l$ between $\AA_{l-1}^{|\Delta|}$ and $\ZZ_l$. The collection of these transformations is an affine space in its own right which we denote by $\WW_l$ and refer to as the layerwise weight space. If we apply $\omega_l$ pointwise, this can be extended to a map 
    \[
    \AA_{l-1}^{|\Delta|\otimes|\TT|} \to \ZZ_l^{|\TT|}.
    \]
\end{itemize}
The computation for this abstract layer is
\begin{align*}
\zeta_l^{(t)} & = \omega_l(\alpha_{l-1}^{(:,t)}) \\
\alpha_l & = \rho_l(\zeta_l^{(t)}),
\end{align*}
where $\rho_l$ is a fixed nonlinear activation function and $\alpha_l$ are the $l$-th layer local activations defined in exactly the same manner as $\alpha_{l-1}$. 

We choose affine bases on $\AA_{l-1}$, $\ZZ_l$, and $\AA_l$. A basis on $\AA_{l-1}$ naturally induces a basis for $\AA_{l-1}^{|\Delta|}$. Consequently, this gives a basis also for the layerwise parameter space $\WW_l$. Let $\iota$, $\kappa$ be two such choices. With respect to $\iota$, we write
\[
\llbracket \alpha_{l-1}^{(:,t)}\rrbracket_{\iota} = \ba_{l-1}^{(:,t)},\  \llbracket\zeta_l^{(t)}\rrbracket_{\iota} = \bz_l^{(t)},\ \llbracket\omega_l\rrbracket_{\iota} = (\bW_l\ \bb_l),\ \llbracket\rho_l\rrbracket_{\iota} = \phi_l,\ \llbracket\alpha_l\rrbracket_{\iota} = \ba_l,
\]
and with respect to $\kappa$, we write
\[
\llbracket \alpha_{l-1}^{(:,t)}\rrbracket_{\kappa} = (\ba_{l-1}^{(:,t)})^{\ddagger},\  \llbracket\zeta_l^{(t)}\rrbracket_{\kappa} = (\bz_l^{(t)})^{\ddagger},\ \llbracket\omega_l\rrbracket_{\kappa} = (\bW_l^{\ddagger}\ \bb_l^{\ddagger}),\ \llbracket\rho_l\rrbracket_{\kappa} = \phi_l^{\ddagger},\ \llbracket\alpha_l\rrbracket_{\kappa} = \ba_l^{\ddagger}.
\]
Note that $\ba_{l-1}^{(:,t)}$ are $J|\Delta|$-dimensional column vectors of the expanded activations matrix $\bA_{l-1}^{\exp}$, $\bz_l^{(t)}$ are $I$-dimensional column vectors of pre-activations matrix $\bZ_l$ and $\ba_l$ are $J$-dimensional column vectors of activations matrix $\bA_l$.

Now, suppose that $(\bOmega_{l-1}\ \bgamma_{l-1})$ is the change-of-basis from $\iota$ to $\kappa$ on $\AA_{l-1}$ and $(\bPhi_l\ \btau_l)$ is the change-of-basis from $\kappa$ to $\iota$ on $\ZZ_l$. If we denote  
\[
[\bOmega_{l-1}]_H = \left[\begin{array}{cc}
\bOmega_{l-1} & \bgamma_{l-1} \\
\mathbf{0} & 1
\end{array}\right],\ 
[\bPhi_l]_H = \left[\begin{array}{cc}
\bPhi_l & \btau_l \\
\mathbf{0} & 1
\end{array}\right],
\]
then by the affine change-of-basis formula for direct products (Eqn.~\ref{eq:change of affine basis product space}), $\mathbf{I}\otimes[\bOmega_{l-1}]_H$ defines the change-of-basis from $\iota$ to $\kappa$ on $\AA_{l-1}^{|\Delta|}$. The parameters $[\bW_l]_H$ and $[\bW_l^{\ddagger}]_H$ are related as follows:
\[
[\bW_l]_H = [\bPhi_l]_H [\bW_l^{\ddagger}]_H(\mathbf{I}\otimes[\bOmega_{l-1}]_H)
\]
By taking $\bm{\Upsilon}_{l-1}$ and $\bm{\Gamma}_l$ in Eqn.~\ref{eq:transformed convolution layer parameters} to be $\bm{\Upsilon}_{l-1}=[\bOmega_{l-1}]_H$ and $\bm{\Gamma}_l=[\bPhi_l]_H$, we can conclude that a convolution layer and its transformed version simply correspond to two different choices of parameterizations for the same underlying abstract convolution layer.

\subsection{Kronecker Factors for Convolution}~\mbox{}

We review the Kronecker Factors for Convolution method~\cite{KFC} which is a version of K-FAC for convolutional networks. The network architecture to consider is a convolutional network with $L$ convolution layers. First, let $\bw$ be the concatenation of all trainable parameters $\bar{\bW}_l$,
\[
\bw = [\ovec(\bar{\bW}_1)^{\top}\ \ovec(\bar{\bW}_2)^{\top}\ \dots\ \ovec(\bar{\bW}_L)^{\top}]^{\top}.
\]
For an input-target pair $(\bx,\by)$, the Fisher matrix for this network is
\[
\bF(\bw) = \EE_{\bx,\by}[(\DD\bw)(\DD\bw)^{\top}],
\]
where $\DD\bw$ is the log-likelihood gradient and the expectation is taken over the model's predictive distribution $P_{\by|\bx}(\bw)$ for $\by$ and over the data distribution for $\bx$. The diagonal blocks $\bF(\bw)_{l,l}$ of $\bF(\bw)$ are
\[
\bF(\bw)_{l,l} = \EE[\ovec(\DD\bar{\bW}_l)\ovec(\DD\bar{\bW}_l)^{\top}].
\]

We are ready now to present the K-FAC approximation for convolutional networks. For a particular layer $l$, we define the K-FAC approximation $\hat{\bF}(\bw)_{l,l}$ to $\bF(\bw)_{l,l}$ as:
\begin{equation} \label{eq:KFC approximation}
\hat{\bF}(\bw)_{l,l} = |\TT|(\EE_{\TT}[\bar{\ba}_{l-1}^{(:,t)}(\bar{\ba}_{l-1}^{(:,t)})^{\top}]\otimes\EE_{\TT}[\DD\bz_l^{(t)}(\DD\bz_l^{(t)})^{\top}]),
\end{equation}
where $\bar{\ba}_{l-1}^{(:,t)}$ is the homogeneous notation for $\ba_{l-1}^{(:,t)}$. The K-FAC approximation $\hat{\bF}(\bw)$ to $\bF(\bw)$ is the matrix with diagonal blocks $\hat{\bF}(\bw)_{l,l}$ as given above and zeros everywhere else. 

Finally, for an objective function $h(\bw)$ defined over the weights, K-FAC optimizes $h(\bw)$ through the following update rule,
\begin{equation} \label{eq:KFC update rule}
\bw\leftarrow\bw-\epsilon\hat{\bF}(\bw)^{-1}\nabla h(\bw).
\end{equation}

\begin{Remark}
Unlike MLPs where K-FAC is derived from assuming only the statistical independence of activations and pre-activation derivatives, convolution layers admit weight sharing and additional assumptions are necessary to derive the approximation $\hat{\bF}(\bw)_{l,l}$ in Eqn.~\ref{eq:KFC approximation}. We refer to~\cite{KFC} for extensive details on how these approximations are derived. 
\end{Remark}
 
Since the purpose of our paper is to derive invariance properties of K-FAC through coordinate-free constructions, we refer the reader to~\cite{KFC} for other aspects of the K-FAC algorithm on convolutional networks, such as implementation details and experimental results. To end our discussion of K-FAC on convolutional networks in the coordinate-dependent case, we present the following theorem which shows that K-FAC is invariant under change-of-basis transformations given in Section~\ref{subsubsec:abstract conv layer}.

\begin{Theorem}[Theorem 3 in~\cite{KFC}]\label{KFC invariances theorem}
Let $\NN$ be a convolutional network with parameter vector $\bw$ and activation functions $\{\phi_l\}^L_{l=1}$. Suppose that we have activation functions $\{\phi_l^\dagger\}^L_{l=1}$ which are related to $\{\phi_l\}^L_{l=1}$ by standard change-of-basis transformations. Then, there exists a parameter vector $\bw^{\dagger}$ such that the transformed network $\NN^{\dagger}$ with parameter vector $\bw^{\dagger}$ and activation functions $\{\phi_l\}^L_{l=1}$ computes the same function as $\NN$. Furthermore, the K-FAC updates are equivalent, in the sense that the resulting networks compute the same function.
\end{Theorem}

The proof of this theorem in~\cite{KFC} again depends on a choice of coordinates. In the next part of our paper, we instead take an intrinsic approach and prove this theorem as a straightforward application of the results given in Section~\ref{sec:abstract-mlp}.

\subsection{Coordinate-free K-FAC for convolutional networks}~\mbox{} \label{subsec:abstract KFC proof}

We begin by considering an abstract convolutional network with $L$ convolution layers.  Let $\XX$ and $\YY$ denote the input and output spaces of this network respectively. Recall that the layerwise weight space $\WW_l$ is the space of affine transformations between $\AA_{l-1}^{|\Delta|}$ and $\ZZ_l$. The weight space of this network is the direct product of all layerwise weight spaces 
\[
\WW=\WW_1\times\dots\times\WW_L.
\]  

Given an input $\xi\in\XX$ and parameter $\omega=(\omega_1,\dots,\omega_L)\in\WW$, denote the network output by $f(\xi,\omega)$. Now, for every $l\in\{1,\dots,L\}$, define the following maps
\begin{itemize}
	\item $\psi_{\xi}^l:\WW_l\to\ZZ_l^{|\TT|}$ which sends layerwise parameters $\omega_l$ to  pre-activations $\zeta_l^{(:)}$ by evaluating local activations $\alpha_{l-1}$ across every spatial location $t\in\TT$
	\item $\varphi_{\xi}^l:\ZZ_l^{|\TT|}\to\YY$ which sends pre-activations $\zeta_l^{(:)}$ to $f(\xi,\omega)$
\end{itemize}
Again, $\psi_{\xi}^l$ is trivially a smooth map from its definition. The map $\varphi_{\xi}^l$ includes all operations in convolutional networks such as max-pooling and response normalization. We make an assumption here that all these operations are smooth. (While this is not the case for common operations such as ReLU and max-pooling, we conjecture that the non-smooth case can be addressed by taking limits of smooth functions.) Finally, we define the map $\Psi_{\xi}^l:\WW_l\to\YY$ as the composition $\Psi_{\xi}^l=\varphi_{\xi}^l\circ\psi_{\xi}^l$.

Let $g$ be a metric on $\YY$ and consider the pullback $(\varphi_{\xi}^l)^*g$ restricted to a single spatial location $t$ which we denote by $(\varphi_{\xi,t}^l)^*g$. More concretely, this metric is computed by assuming components of the tangent vector at all other spatial locations are zero. Now, let us take $A$, $Z$, $W$ in Section~\ref{subsec:independence metric} to be
\[
A=\AA_{l-1}^{|\Delta|},\ Z=\ZZ_l,\ W=\WW_l,
\] 
and the metric on $Z=\ZZ_l$ to be $(\varphi_{\xi,t}^l)^*g$. Summing over every spatial location $t\in\TT$, the independence metric on $\WW_l$ here is
\begin{equation} \label{eq:KFC metric components}
g_{\mathrm{ind}}^l(\omega_l) = |\TT|(\EE_{\TT}[\alpha_{l-1}^{(:,t)}\otimes\alpha_{l-1}^{(:,t)}]\otimes\EE_{\TT}[(\varphi_{\xi,t}^l)^*g(\zeta_l^{(t)})]).
\end{equation}

\begin{Definition} \label{KFC metric definition}
The K-FAC metric on the weight space $\WW$ of an abstract convolutional network is defined as
\[
g_{\mathrm{KFAC}}(\omega) = g^1_{\mathrm{ind}}(\omega_1) + \dots + g^L_{\mathrm{ind}}(\omega_L),
\]
where the sum above is as defined in Eqn.~\ref{eq:product metric on tangent vectors} and each $g_{\mathrm{ind}}^l$ is as given in Eqn.~\ref{eq:KFC metric components}.
\end{Definition}

\begin{Theorem} \label{abstract KFC theorem}
Let $g$ be a metric on $\YY$. Then, $g_{\mathrm{KFAC}}$ given in Definition~\ref{KFC metric definition} is indeed a metric on the weight space $\WW$ of an abstract convolutional network. Moreover, if we assume that the expected pullback of $g$ restricted to a single spatial location $t\in\TT$,
\[
\EE_{\xi}[(\Psi_{\xi,t}^l)^*g],
\]
under the map $\Psi_{\xi}^l:\WW_l\to\YY$ is a nondegenerate metric on the layerwise weight space $\WW_l$ for every $l$, then $g_{\mathrm{KFAC}}$ is also nondegenerate. 
\end{Theorem}

\begin{proof}
The proof of this theorem mirrors the proof given earlier for Theorem~\ref{abstract KFAC theorem for MLP}. By Theorem~\ref{general KFAC theorem}, we know that
\[
\EE[\alpha_{l-1}^{(:,t)}\otimes\alpha_{l-1}^{(:,t)}]\otimes\EE[(\varphi_{\xi,t}^l)^*g(\zeta_l^{(t)})],
\]
is a metric on $\WW_l$. Since taking expectation over the set of spatial locations $\TT$ and multiplying by the scale factor $|\TT|$ preserves the metric properties, we obtain that $g_{\mathrm{ind}}^l$ in Eqn.~\ref{eq:KFC metric components} defines a metric on $\WW_l$. Consequently, $g_{\mathrm{KFAC}}$ determines a metric on $\WW$. To prove the latter assertion, note that by the functorial property of pullback operations, 
\[
\EE_{\xi}[(\psi_{\xi,t}^l)^*((\varphi_{\xi,t}^l)^*g)]
\]
is nondegenerate. Using the second assertion of Theorem~\ref{general KFAC theorem} yields that $g_{\mathrm{ind}}^l$ is nondegenerate which implies that this is true also for $g_{\mathrm{KFAC}}$.
\end{proof}

We conclude this section with a proof of Theorem~\ref{KFC invariances theorem}. Our proof is coordinate-free and given in exactly the same manner as the proof of Theorem~\ref{KFAC invariances theorem} at the end of Section~\ref{sec:abstract-mlp}.

\vspace{.05in}

\noindent {\bf {Coordinate-free proof of Theorem~\ref{KFC invariances theorem}}.}  As shown earlier in Section~\ref{subsubsec:abstract conv layer}, each convolution layer of $\NN$ and $\NN^{\dagger}$ correspond to two different choices of parameterizations for the same underlying abstract convolution layer. Hence, $\NN$ and $\NN^{\dagger}$ must compute the same function. 

Assume that the metric $g$ on the output space $\YY$ in Theorem~\ref{abstract KFC theorem} is the output Fisher metric $g_{F,\mathrm{out}}$. For each spatial location $t\in\TT$, the pullback under $\varphi_{\xi,t}^l$ is
\[
(\varphi_{\xi,t}^l)^*g_{F,\mathrm{out}}(\zeta_l^{(t)}) = \EE[d\LL_{\zeta_l^{(t)}}\otimes d\LL_{\zeta_l^{(t)}}].
\]
Now, choose coordinate systems on $\AA_{l-1}$ and $\ZZ_l$. This induces coordinates for $\AA_{l-1}^{|\Delta|}$ and we write
\[
\llbracket\alpha_{l-1}^{(:,t)}\rrbracket = \ba_{l-1}^{(:,t)},\ \llbracket\zeta_l^{(t)}\rrbracket = \bz_l^{(t)},\ \llbracket(\varphi_{\xi,t}^l)^*g_{F,\mathrm{out}}(\zeta_l^{(t)})\rrbracket = \llbracket\EE[d\LL_{\zeta_l^{(t)}}\otimes d\LL_{\zeta_l^{(t)}}]\rrbracket = \EE[\DD\bz_l^{(t)}(\DD\bz_l^{(t)})^{\top}].
\]
Using Proposition~\ref{general K-FAC metric coordinates prop}, the independence metric in Eqn.~\ref{eq:KFC metric components} can be expressed in coordinates as follows
\[
\llbracket g_{\mathrm{ind}}^l(\omega_l)\rrbracket = |\TT|(\EE_{\TT}[\bar{\ba}_{l-1}^{(:,t)}(\bar{\ba}_{l-1}^{(:,t)})^{\top}]\otimes\EE_{\TT}[\DD\bz_l^{(t)}(\DD\bz_l^{(t)})^{\top}]),
\]
which is exactly $\hat{\bF}(\bw)_{l,l}$ given earlier in Eqn.~\ref{eq:KFC approximation}. Furthermore, $\llbracket g_{\mathrm{KFAC}}(\omega)\rrbracket$ is exactly the K-FAC approximation $\hat{\bF}(\bw)$. Thus, the K-FAC update rule in Eqn.~\ref{eq:KFC update rule} is simply a natural gradient update rule with respect to the K-FAC metric $g_{\mathrm{KFAC}}$ for abstract convolutional networks. Lastly, if $g_{\mathrm{KFAC}}$ is a nondegenerate metric; which is true for example if the assumptions in the second assertion of Theorem~\ref{abstract KFC theorem} hold, then we can conclude that these updates are invariant to any affine reparameterizations of the model. \qed

\section{coordinate-free k-fac for recurrent networks} \label{sec:abstract-RNN}

In this section, we study Kronecker factorization for recurrent networks closely following~\cite{martens2018kronecker}. We give a mathematical formulation of the recurrent computation step of these networks in both coordinate-dependent and coordinate-independent scenarios. We proceed to give the Kronecker factorization of the Fisher matrix for recurrent networks and then state the invariance theorem for this optimization method. Lastly, we prove the invariance theorem in the same way we did for MLPs and convolutional networks in Sections~\ref{sec:abstract-mlp} and~\ref{sec:abstract-CNN} respectively. 

\subsection{Recurrent networks}~\mbox{}

As in the case of convolutional networks in Section~\ref{sec:abstract-CNN}, it is not necessary to write out the full structure of a recurrent network. Rather, we focus on the recurrent computation since the central object of our interest, the Fisher matrix for recurrent networks, only involves recurrent weights.

\subsubsection{Computational step}~\mbox{}

Let $T$ be the number of different time steps and $\bT=\{1,\dots,T\}$. We use $t$ to index the time step. Throughout, we assume that all sequences are of fixed length $T$. For an input $\bx_t$ at every $t$, the recurrent network maps this to an output $\bo_t$. Essentially, the network maps input sequences $\bx=(\bx_1,\dots,\bx_T)$ to output sequences $\bo=(\bo_1,\dots,\bo_T)$. The computation, at every $t$, is
\begin{align*}
\bz_t & = \bW\ba_{t-1} + \bb \\
\bz_t' & = \bz_t + \mathbf{V}\bx_t \\
\ba_t & = \phi(\bz_t')
\end{align*}
where $\ba_{t-1}$ is an activation vector, $\bz_t$ is a pre-activation vector, $\bW$ is a recurrent weight matrix, $\mathbf{V}$ is a weight matrix, $\bb$ is a recurrent bias vector, and $\phi$ is a fixed nonlinear activation function. For the remainder of this section, we focus on the first equation 
\begin{equation} \label{eq:recurrent computation}
\bz_t = \bW\ba_{t-1} + \bb,
\end{equation}
which represents the recurrent computation step. The latter two equations can be handled by the previous K-FAC analysis for MLPs given in Section~\ref{sec:abstract-mlp}. The transformed recurrent computation step is defined as
\begin{equation} \label{eq:transformed recurrent computation}
\bz_t^{\dagger} = \bW^{\dagger}\ba_{t-1}^{\dagger} + \bb^{\dagger}.
\end{equation}
The relationship between transformed parameters $(\bW^{\dagger}\ \bb^{\dagger})$ and original parameters $(\bW\ \bb)$ is given by a standard change-of-basis formula as in Eqn.~\ref{eq:change of basis for parameters}.

\subsubsection{Abstract recurrent network}~\mbox{} \label{subsubsec:abstract RNN}

We now describe an abstract recurrent network formally. 
\begin{itemize}
\item Local activations at each time step $t$ are elements $\alpha_t$ in an affine space $\AA$
\item Activations are elements $\alpha=\{\alpha_t\}_{t\in\bT}$ in the affine space $\AA^T$
\item Local pre-activations at each $t$ are elements $\zeta_t$ in an affine space $\ZZ$
\item Pre-activations are elements $\zeta=\{\zeta_t\}_{t\in\bT}$ in the affine space $\ZZ^T$
\item Parameters are affine transformations $\omega$ between $\AA$ and $\ZZ$. The collection of these transformations is an affine space in its own right which we denote by $\WW$ and refer to as the weight space
\item Network inputs and outputs at each $t$ are elements $\xi_t$, $\upsilon_t$ in affine spaces $\XX$, $\YY$ respectively. The input and output spaces are $\XX^T$ and $\YY^T$ respectively; furthermore, elements here are written as $\xi=\{\xi_t\}_{t\in\bT}$ and $\upsilon=\{\upsilon_t\}_{t\in\bT}$.
\end{itemize}
For every $t$, the abstract recurrent computation step is
\[
\zeta_t = \omega(\alpha_{t-1}).
\]

A choice of parameterization for the abstract recurrent network consists of choosing affine bases for $\AA$, $\ZZ$, $\XX$ and $\YY$. Since we have bases for $\AA$ and $\ZZ$, this induces a natural basis for $\WW$. If we use exactly the same change-of-basis analysis given in Section~\ref{sec:abstract-mlp}, then the recurrent network with computation given by Eqn.~\ref{eq:recurrent computation} and the transformed version in Eqn.~\ref{eq:transformed recurrent computation} correspond to two different parameterizations of the same abstract recurrent network. 

\subsection{K-FAC for recurrent networks}~\mbox{}

We review the recent Kronecker factorization for recurrent networks method in~\cite{martens2018kronecker}. Recall that for every time step $t$, the recurrent computation can be written as
\[
\bz_t = \bar{\bW}\bar{\ba}_{t-1},
\]
where $\bar{\bW}=[\bW\ \bb]$ and $\bar{\ba}_{t-1}^{\top}=[\bar{\ba}_{t-1}^{\top}\ 1]^{\top}$. Using backpropagation, the log-likelihood gradient is given by $\DD\bz_t\bar{\ba}_{t-1}^{\top}$. The total contribution to the gradient across all $t$ is the sum
\[
\DD\bar{\bW} = \sum^T_{t=1}\DD\bz_t\bar{\ba}_{t-1}^{\top}.
\]
For an input-target pair $(\bx,\by)$, the Fisher matrix $\bF(\bar{\bW})$ for recurrent networks is defined as
\[
\bF(\bar{\bW}) = \EE_{\bx,\by}[\ovec(\DD\bar{\bW})\ovec(\DD\bar{\bW})^{\top}].
\]
Finally, the K-FAC approximation $\hat{\bF}(\bar{\bW})$ to $\bF(\bar{\bW})$ for recurrent networks is defined as
\begin{equation} \label{eq:KFAC approximation RNN}
\hat{\bF}(\bar{\bW}) = T(\EE_{\bT}[\bar{\ba}_{t-1}\bar{\ba}_{t-1}^{\top}]\otimes\EE_{\bT}[\DD\bz_t\DD\bz_t^{\top}]).
\end{equation}

\begin{Remark}
As in the case of convolution layers, there is weight sharing in recurrent networks (across time here instead of spatial locations) and so it is not enough to just assume statistical independence between activations and pre-activation derivatives to make the K-FAC approximation here. We defer the reader to~\cite{martens2018kronecker} for detailed explanations on how the K-FAC approximation is derived for recurrent networks.
\end{Remark}

For an objective function $h(\bar{\bW})$ on the weight space of the recurrent network, K-FAC minimizes $h(\bar{\bW})$ by the update rule
\begin{equation} \label{eq:KFAC-RNN update rule}
\bar{\bW}\leftarrow\bar{\bW}-\epsilon\hat{\bF}(\bar{\bW})^{-1}\nabla h(\bar{\bW}).
\end{equation}
Lastly, we present the invariance theorem for K-FAC on recurrent networks. 

\begin{Theorem} \label{KFAC-RNN invariances theorem}
Let $\NN$ be a recurrent network with recurrent parameters $(\bW\ \bb)$. Suppose that we have a recurrent network $\NN^{\dagger}$ with recurrent parameters $(\bW^{\dagger}\ \bb^{\dagger})$ and the relationship between $(\bW\ \bb)$ and $(\bW^{\dagger}\ \bb^{\dagger})$ is a change-of-basis transformation as given in Eqn.~\ref{eq:change of basis for parameters}. Then, the networks $\NN$ and $\NN^{\dagger}$ compute the same function. Furthermore, the K-FAC updates are equivalent, in the sense that the resulting networks compute the same function.
\end{Theorem}

We now proceed to the last section of this paper to give a coordinate-free proof of this theorem. The method of proof mirrors exactly the proofs given previously for Theorems~\ref{KFAC invariances theorem} and~\ref{KFC invariances theorem}.

\subsection{Coordinate-free K-FAC for recurrent networks}~\mbox{}

Given an input $\xi=\{\xi_t\}_{t\in\bT}\in\XX^T$ and parameter $\omega\in\WW$, denote the network output by $f(\xi,\omega)$. For a specific time step $t$, consider the following maps:
\begin{itemize}
	\item $\psi_{\xi,t}:\WW\to\ZZ$ which sends parameters $\omega$ to pre-activations $\zeta_t$ by evaluation at activations $\alpha_{t-1}$
	\item $\varphi_{\xi,t}:\ZZ\to\YY$ which sends $\zeta_t$ to outputs 
\end{itemize}
In addition, we define the map $\Psi_{\xi,t}:\WW\to\YY$ as the composition $\Psi_{\xi,t}=\varphi_{\xi,t}\circ\psi_{\xi,t}$.

Let $g$ be a metric on $\YY$. The pullback $\varphi_{\xi,t}^*g$ then defines a metric on $\ZZ$. Now, we take $A$, $Z$, $W$ in Section~\ref{subsec:independence metric} to be
\[
A=\AA,\ Z=\ZZ,\ W=\WW,
\]
and the metric on $Z=\ZZ$ to be $\varphi_{\xi,t}^*g$. Summing over all time steps $t\in\bT$, we make the following definition which arises from the independence metric in Section~\ref{subsec:KFAC metric}:

\begin{Definition} \label{KFAC-RNN metric definition}
The K-FAC metric on the weight space $\WW$ of an abstract recurrent network is defined as
\begin{equation} \label{eq:KFAC-RNN metric decomposition}
g_{\mathrm{KFAC}}(\omega) = T(\EE_{\bT}[\alpha_{t-1}\otimes\alpha_{t-1}]\otimes\EE_{\bT}[\varphi_{\xi,t}^*g(\zeta_t)]).
\end{equation}
\end{Definition}

\begin{Theorem} \label{abstract KFAC-RNN theorem}
Let $g$ be a metric on $\YY$. Then, $g_{\mathrm{KFAC}}$ given in Definition~\ref{KFAC-RNN metric definition} is a metric on the weight space $\WW$ of an abstract recurrent network. Moreover, if we assume that the expected pullback of $g$,
\[
\EE_{\xi}[\Psi_{\xi,t}^*g],
\]
under the smooth map $\Psi_{\xi,t}:\WW\to\YY$ is a nondegenerate metric, then $g_{\mathrm{KFAC}}$ is also nondegenerate.
\end{Theorem}

\begin{proof}
The proof of this theorem is analogous to the proofs of Theorems~\ref{abstract KFAC theorem for MLP} and~\ref{abstract KFC theorem}. From Theorem~\ref{general KFAC theorem}, we have that
\[
\EE[\alpha_{t-1}\otimes\alpha_{t-1}]\otimes\EE[\varphi_{\xi,t}^*g(\zeta_t)]
\]
is a metric on $\WW$. Since this remains true after taking expectation over the set of time steps $\bT$ and multiplying by the scale factor $T$, we can conclude that $g_{\mathrm{KFAC}}$ is a metric on $\WW$. For the nondegeneracy statement, using the functorial property of pullbacks, we know that
\[
\EE_{\xi}[\psi_{\xi,t}^*(\varphi_{\xi,t}^*g)]
\]
is nondegenerate. Then, $g_{\mathrm{KFAC}}$ is nondegenerate by the second assertion of Theorem~\ref{general KFAC theorem}.
\end{proof}

\vspace{.05in}

\noindent {\bf {Coordinate-free proof of Theorem~\ref{KFAC-RNN invariances theorem}}.} 
We conclude this paper with a coordinate-free proof of Theorem~\ref{KFAC-RNN invariances theorem}. As mentioned at the end of Section~\ref{subsubsec:abstract RNN}, $\NN$ and $\NN^{\dagger}$ correspond to two different choices of parameterizations for the same underlying abstract recurrent network and so they must compute the same function.

Assume that the metric $g$ on $\YY$ in Theorem~\ref{abstract KFAC-RNN theorem} above is the output Fisher metric $g_{F,\mathrm{out}}$. Then, the pullback under $\varphi_{\xi,t}$ is
\[
\varphi_{\xi,t}^*g_{F,\mathrm{out}}(\zeta_t) = \EE[d\LL_{\zeta_t}\otimes d\LL_{\zeta_t}].
\]
Now, choose coordinate systems for $\AA$ and $\ZZ$. We can write
\[
\llbracket\alpha_{t-1}\rrbracket = \ba_{t-1},\ \llbracket\zeta_t\rrbracket = \bz_t,\ \llbracket\varphi_{\xi,t}^*g_{F,\mathrm{out}}(\zeta_t)\rrbracket = \EE[\DD\bz_t\DD\bz_t^{\top}].
\]
By Proposition~\ref{general K-FAC metric coordinates prop}, the K-FAC metric in Eqn.~\ref{eq:KFAC-RNN metric decomposition} can be represented in these chosen coordinates as
\[
\llbracket g_{\mathrm{KFAC}}(\omega)\rrbracket =  T(\EE_{\bT}[\bar{\ba}_{t-1}\bar{\ba}_{t-1}^{\top}]\otimes\EE_{\bT}[\DD\bz_t\DD\bz_t^{\top}]),
\]
which is exactly the K-FAC approximation $\hat{\bF}(\bar{\bW})$ in Eqn.~\ref{eq:KFAC approximation RNN}. Thus, the K-FAC update rule in Eqn.~\ref{eq:KFAC-RNN update rule} is a natural gradient update with respect to the K-FAC metric $g_{\mathrm{KFAC}}$ for abstract recurrent networks. If we suppose that $g_{\mathrm{KFAC}}$ is nondegenerate, then these updates are invariant to any affine reparameterizations of the model. \qed 

\bibliographystyle{plain}
\bibliography{bibliography}

\newpage
\appendix

\section{appendix}

\subsection{Tangent space of vector spaces and affine spaces}~\mbox{}

\begin{Theorem} 
Let $V$ be a finite-dimensional vector space. For each point $p\in V$, there is a canonical isomorphism $V\to T_pV$. From now on, we suppress $p$ in $T_pV$ and write $TV$ when denoting tangent spaces of $V$.
\end{Theorem}

\begin{proof}
For any element $v\in V$, we can associate a tangent vector $D_v|_p$ at $p$ defined by
\[
D_v|_p f=D_v f(p) = \frac{d}{dt}\bigg\rvert_{t=0}f(p+tv),
\]
where $f$ is a smooth function on $V$. This gives the desired canonical isomorphism since the above construction involves no choice of basis.
\end{proof}

\begin{Corollary} \label{tangent spaces of affine spaces corollary}
Let $A$ be an affine space and $V$ be its associated vector space. For each point $a\in A$, there is a canonical isomorphism between $T_aA$ and $V$. From now on, we suppress $a$ in $T_aA$ and write $TA$ when denoting tangent spaces of $A$.
\end{Corollary}

\begin{proof}
Note that specifying a point $a\in A$ naturally identifies $A$ with $V$. Then, applying the above theorem gives the desired result.
\end{proof}

\subsection{Tensor product of vector spaces}~\mbox{} \label{Asubsec: tensor product of VS}

Let $U$ and $V$ be finite-dimensional vector spaces over the real numbers $\RR$. Let $\mathcal{R}$ be the subspace of the free vector space $\RR\langle U\times V\rangle$ (set of all finite formal linear combinations of elements of $U\times V$ with real coefficients) spanned by all elements of the following forms:
\begin{align*}
c(u,v) & - (cu,v), \\
c(u,v) & - (u,cv), \\
(u,v)+(u',v) & - (u+u',v), \\
(u,v)+(u,v') & - (u,v+v'),
\end{align*}
for $u,u'\in U$, $v,v'\in V$, and $c\in\RR$. The tensor product $U\otimes V$, is the quotient space $\frac{\mathbb{R}\langle U\times V\rangle}{\mathcal{R}}$ and the equivalence class of an element $(u,v)$ in $U\otimes V$ is denoted by $u\otimes v$. 

We describe how the vector space of linear transformations between $U$ and $V$, denoted by $\mathrm{Hom}(U,V)$, may be thought of as tensor products. There is a canonical isomorphism
\begin{equation} \label{eq:Hom-tensor isomorphism}
U^*\otimes V\to\mathrm{Hom}(U,V)
\end{equation}
given by $\varphi\otimes v\mapsto\varphi(u)v$ where $\varphi\in U^*$. Another isomorphism of interest to us is 
\begin{equation} \label{eq:tensor product of duals}
U^*\otimes V^*\to (U\otimes V)^*,
\end{equation}
which is again canonical. To derive this isomorphism, given $\varphi\in U^*$, $\phi\in V^*$, consider the bilinear map $U\times V\to\RR$ defined by
\[
(u,v)\mapsto\varphi(u)\cdot\phi(v).
\]
This induces an element on the tensor product $(U\otimes V)^*$. As such, we obtain an unique linear injection
\[
U^*\otimes V^*\to (U\otimes V)^*.
\]
Since all the vector spaces are finite-dimensional, we can conclude that this is an isomorphism.

We now use these facts to explain several ingredients in the proof of Theorem~\ref{subsec:KFAC metric} in greater detail. The element $a$ is a linear map $TW\to TZ$ and by the above isomorphism, this means $a\in T^*W\otimes TZ$ from Eqn.~\ref{eq:Hom-tensor isomorphism}. By Eqn.~\ref{eq:tensor product of duals}, the dual space of $(T^*W\otimes TZ)^*\cong TW\otimes T^*Z$. Using Eqn.~\ref{eq:Hom-tensor isomorphism} again, elements of this space are maps $\lambda:T^*W\to T^*Z$.

\subsection{Derivations of Eqns.~\ref{eq:change of basis for parameters} and~\ref{eq:change of basis for activation functions}}~\mbox{} \label{Asubsec:change of affine basis computations}~

We now provide a derivation of the equalities in Eqns.~\ref{eq:change of basis for parameters} and~\ref{eq:change of basis for activation functions} in Section~\ref{subsec:coordinate-free network setup}. Consider the following commutative diagram (the top horizontal arrow is equal to the composition of maps given by the other three arrows) which relates the two parameterizations $\iota$ and $\kappa$ on the activation affine space $\AA_{i-1}$ and the pre-activation affine space $\ZZ_i$:
\begin{equation}	\label{eq:comm diagram I}
\xymatrix{
	\llbracket \AA_{i-1}\rrbracket_{\iota}\ar[r]^{(\mathbf{W}_{i}\ \mathbf{b}_{i})}\ar[d]_{(\mathbf{\Omega}_{i-1},\bm{\gamma}_{i-1})} & \llbracket \ZZ_{i}\rrbracket_{\iota} \\
	\llbracket \AA_{i-1}\rrbracket_{\kappa}\ar[r]^{(\mathbf{W}_{i}^{\ddagger}\ \mathbf{b}_{i}^{\ddagger})} & \llbracket \ZZ_{i}\rrbracket_{\kappa}\ar[u]_{(\mathbf{\Phi}_{i},\mathbf{\tau}_{i})} }
\end{equation}
Let $\mathbf{a}_{i-1}\in \llbracket \AA_{i-1}\rrbracket_{\iota}$, this maps to $\mathbf{z}_{i}=\mathbf{W}_{i}\mathbf{a}_{i-1}+\mathbf{b}_{i}\in\llbracket \ZZ_{i}\rrbracket_{\iota}$ under the top horizontal arrow in Eqn.~\ref{eq:comm diagram I}. Now, mapping $\mathbf{a}_{i-1}$ under the composition of the other three arrows in Eqn.~\ref{eq:comm diagram I}, we obtain
\begin{align*}
\mathbf{a}_{i-1} & \mapsto\mathbf{\Omega}_{i}\mathbf{a}_{i-1}+\bm{\gamma}_{i-1} && \text{(apply left vertical arrow in Eqn.~\ref{eq:comm diagram I})} \\
& \mapsto \mathbf{W}_{i}^{\ddagger}(\mathbf{\Omega}_{i}\mathbf{a}_{i-1}+\bm{\gamma}_{i-1})+\mathbf{b}_{i}^{\ddagger} && \text{(apply bottom horizontal arrow in Eqn.~\ref{eq:comm diagram I})} \\
& \mapsto \mathbf{\Phi}_{i}(\mathbf{W}_{i}^{\ddagger}\mathbf{\Omega}_{i}\mathbf{a}_{i-1}+\mathbf{W}_{i}^{\ddagger}\bm{\gamma}_{i-1})+\mathbf{\Phi}_{i}\mathbf{b}_{i}^{\ddagger}+\bm{\tau}_{i} && \text{(apply right vertical arrow in Eqn.~\ref{eq:comm diagram I})} \\
& = \mathbf{\Phi}_{i}\mathbf{W}_{i}^{\ddagger}\mathbf{\Omega}_{i-1}\mathbf{a}_{i-1} + \mathbf{\Phi}_{i}\mathbf{W}_{i}^{\ddagger}\bm{\gamma}_{i-1} + \mathbf{\Phi_{i}}\mathbf{b}_{i}^{\ddagger} + \bm{\tau}_{i}.
\end{align*}
This establishes the equality given in Eqn.~\ref{eq:change of basis for parameters}. For Eqn.~\ref{eq:change of basis for activation functions}, we use the commutative diagram (this time, the bottom horizontal arrow is equal to the composition of the other three arrows):
\begin{equation}	\label{eq:comm diagram II}
\xymatrix{
	\llbracket \ZZ_{i}\rrbracket_{\iota}\ar[r]^{\phi_{i}} & \llbracket \AA_{i}\rrbracket_{\iota}\ar[d]^{(\mathbf{\Omega}_{i},\bm{\gamma}_{i})} \\
	\llbracket \ZZ_{i}\rrbracket_{\kappa}\ar[r]_{\phi_{i}^{\ddagger}}\ar[u]^{(\mathbf{\Phi}_{i},\bm{\tau}_{i})} & \llbracket \AA_{i}\rrbracket_{\kappa} }
\end{equation}
Let $\mathbf{z}_{i}^{\ddagger}\in\llbracket \ZZ_{i}\rrbracket_{\kappa}$, this maps to $\mathbf{a}_{i}^{\ddagger}\in\llbracket \AA_{i}\rrbracket_{\kappa}$ under $\phi_{i}^{\ddagger}$. Now, mapping $\mathbf{z}_{i}^{\ddagger}$ under the other three arrows in Eqn.~\ref{eq:comm diagram II}, we obtain
\begin{align*}
\mathbf{z}_{i}^{\ddagger} & \mapsto \mathbf{\Phi}_{i}\mathbf{z}_{i}^{\ddagger}+\bm{\tau}_{i} && \text{(apply left vertical arrow in Eqn.~\ref{eq:comm diagram II})} \\
& \mapsto \phi_{i}(\mathbf{\Phi}_{i}\mathbf{z}_{i}^{\ddagger}+\bm{\tau}_{i}). && \text{(apply upper horizontal arrow in Eqn.~\ref{eq:comm diagram II})} \\
& \mapsto \bOmega_i\phi_i(\bPhi_i\bz_i^{\ddagger} + \btau_i) + \bgamma_i && \text{(apply right vertical arrow in Eqn.~\ref{eq:comm diagram II})}.
\end{align*}
This establishes Eqn.~\ref{eq:change of basis for activation functions}.

\end{document}